\documentclass[journal]{IEEEtran}
\ifCLASSINFOpdf
\else
\fi

\usepackage{algorithm}
\usepackage{algorithmic}

\usepackage{amsmath}
\usepackage{amsthm}
\usepackage{amssymb}
\usepackage{graphicx}
\usepackage{epsf}
\usepackage{color}

\usepackage{multirow}

\usepackage{dblfloatfix}
\usepackage{subcaption}
\usepackage{wrapfig}

\usepackage[numbers]{natbib}

\usepackage{mathtools}
\usepackage{mathrsfs}
\mathtoolsset{showonlyrefs=true}

\newcounter{lem_counter}
\newcounter{pro_counter}

\newcounter{def_counter}

\newtheorem{proposition}[pro_counter]{Proposition}
\newtheorem{lemma}[lem_counter]{Lemma}

\newtheorem{definition}[def_counter]{Definition}

\usepackage[finalold]{trackchanges}

\addeditor{lb}
\addeditor{dm}
\addeditor{mg}
\addeditor{wd}

\begin{document}
%

\title{Stable and Efficient Policy Evaluation}

\author{Daoming Lyu,
        Bo Liu,~\IEEEmembership{Member,~IEEE,}
        Matthieu Geist,
        Wen Dong, \\
        Saad Biaz,~\IEEEmembership{Member,~IEEE,}
        and Qi Wang,~\IEEEmembership{Senior Member,~IEEE}
\thanks{Manuscript received January 08, 2018; revised April 10, 2018 and August 5, 2018; accepted September 2, 2018. The work of B. Liu was supported by Tencent Rhino-bird Gift Funds.(\textit{Corresponding author: Bo Liu.})}

\thanks{D. Lyu, B. Liu, and S. Biaz are with the Department of Computer Science and Software Engineering, Auburn University, Auburn, AL 36849 USA (e-mail: daoming.lyu@auburn.edu; boliu@auburn.edu; biazsaa@auburn.edu).}
\thanks{Matthieu Geist is with Universit\'e de Lorraine, CNRS, LIEC, F-57000 Metz, France, and also with Google Brain, 75009 Paris, France(e-mail: matthieu.geist@univ-lorraine.fr).}
\thanks{Wen Dong is with the Department of Computer Science and Engineering, State University of New York at Buffalo, Buffalo, NY 14260 USA (e-mail: wendong@buffalo.edu).}
\thanks{Qi Wang is with the School of Computer Science and Center for OPTical IMagery Analysis and Learning (OPTIMAL), Northwestern Polytechnical University, Xi'an 710072, China(e-mail: crabwq@nwpu.edu.cn).}}

%
%


\markboth{Journal of IEEE Transactions on Neural Networks and Learning Systems}{}


%



\maketitle

\begin{abstract} 
Policy evaluation algorithms are essential to reinforcement learning due to their ability to predict the performance of a policy. However, there are two long-standing issues lying in this prediction problem that need to be tackled: off-policy stability and on-policy efficiency. The conventional temporal difference (TD) algorithm is known to perform very well in the on-policy setting, yet is not off-policy stable. On the other hand, the gradient TD and emphatic TD algorithms are off-policy stable, but are not on-policy efficient. 
%
This paper introduces novel algorithms that are both off-policy stable and on-policy efficient by using the oblique projection method. The empirical experimental results on various domains validate the effectiveness of the proposed approach.
\end{abstract}


%
\IEEEpeerreviewmaketitle


\section{Introduction}
\label{sec:introduction}
\IEEEPARstart{P}{olicy} evaluation plays a crucial role in reinforcement learning ({\bf RL}): it estimates a value function that can predict the long-term return for a given fixed policy. {\em Temporal difference} ({\bf TD}) learning is the central and powerful policy evaluation method in RL. However, it has two fundamental problems. 
The first problem is the \textbf{off-policy stability}. Although TD converges when samples are drawn ``on-policy'' (from the policy to be evaluated),
it is shown to be possibly divergent when samples are drawn ``off-policy''. Off-policy stable methods are of wider interest since they can learn while executing an exploratory policy, learn from demonstrations, and learn multiple tasks in parallel. 
Several different approaches have been explored to address off-policy learning. The ``averager'' method~\cite{gordon1996stable} needs to store many training examples, and thus is not practical for large-scale applications. Off-policy LSTD~\citep{yu2010:icml} is off-policy convergent, but its per-step computational complexity
is quadratic in the number of parameters $d$ of the function approximator. The most state-of-the-art off-policy stable algorithms with linear computational complexity are gradient TD (GTD)~\citep{tdc:2009} and proximal gradient TD (PGTD)~\citep{liu2015uai}, that use stochastic primal-dual based methods as powerful solvers. 
The second problem is the \textbf{on-policy efficiency}. Although GTD and PGTD are off-policy stable, they usually tend to have inferior performances in on-policy learning settings, especially in small-scale problems with relatively few samples~\citep{pgtd:adam2016}. On the other hand, the TD method is well-known for its on-policy efficiency, which explains well its popularity among reinforcement learning researchers and practitioners.
It is intriguing, therefore, to propose model-free policy evaluation algorithms that offer both off-policy stability and on-policy efficiency.

The major contribution of this paper is to explore policy evaluation algorithms that yield both off-policy stability and on-policy efficiency. To this end, we propose novel algorithms based on the oblique projection framework~\citep{Scherrer:ObliqueProjection}. A computationally feasible criterion is proposed and used to derive algorithms with linear computational complexity per step. The off-policy stability is rigorously proved, and the on-policy and off-policy performances are demonstrated via thorough experimental studies.

Here is a roadmap for the rest of the paper. Section~\ref{sec:preliminary} introduces some RL background, reviews existing approaches to tackle the problem of off-policy stability and puts off-policy policy evaluation in the framework of (weighted) oblique projection. Section~\ref{sec:alg} provides the stable and efficient TD (SETD) algorithm and a more general SETD($\lambda$) algorithm using weighted oblique projection. 
Related works are discussed in Section~\ref{sec:related} and compared empirically to the proposed SETD in Section~\ref{sec:experimental}.


\section{Preliminary}
\label{sec:preliminary}



\subsection{Reinforcement Learning}
\label{sec:Rinforcement Learning}
Reinforcement learning~\citep{sutton-barto:book, tnnls:kiumarsi:2018} and approximate dynamic programming~\citep{tnnls:liu:2014policy,tnnls:he:2017:survey}  is a class of learning problems in which an agent interacts with an unfamiliar, dynamical, and stochastic environment, where the agent's goal is to optimize some measure of its long-term performance. This interaction is conventionally modeled as a Markov decision process (\textbf{MDP}). An MDP is defined as the tuple $({\mathcal{S},\mathcal{A},P_{ss'}^{a},R,\gamma})$, where $\mathcal{S}$ and $\mathcal{A}$ are finite sets of states and actions, the transition kernel $P_{ss'}^{a}$ specifies the probability of transition from state $s\in\mathcal{S}$ to state $s'\in\mathcal{S}$ by taking action $a\in\mathcal{A}$, $R(s,a):\mathcal{S}\times\mathcal{A}\to\mathbb{R}$ is the reward function bounded by $R_{\max}$, and $0\leq\gamma<1$ is a discount factor. A stationary policy $\pi:\mathcal{S}\times\mathcal{A}\to\left[{0,1}\right]$ is a probabilistic mapping from states to actions. The main objective of an RL algorithm is to find an optimal policy. In order to achieve this goal, a key step in many algorithms is to estimate the value function under a given policy $\pi$, i.e.,~$V_{\pi}:\mathcal{S}\to\mathbb{R}$, a process known as {\em policy evaluation}. It is known that $V_\pi$ is the unique fixed-point of the {\em Bellman operator} $T_\pi$, i.e.,
\begin{align}
\label{eq:BellmanEq}
V_\pi = T_\pi V_\pi = R_\pi + \gamma P_\pi V_\pi,
\end{align}
where $R_\pi$ and $P_\pi$ are respectively the reward function and transition kernel of the Markov chain induced by policy $\pi$. In Eq.~\eqref{eq:BellmanEq}, we may think of $V_\pi$ as an $|\mathcal{S}|$-dimensional vector and write everything in vector/matrix form.
In the following, to simplify the notation, we often drop the dependence of $T_\pi$, $V_\pi$, $R_\pi$, and $P_\pi$ to $\pi$. 
We denote by $\pi_b$, the behavior policy that generates the data, and by $\pi$, the target policy that we would like to evaluate. They are the same in the on-policy setting and different in the off-policy scenario. For the $i$-th state-action pair $(s_i,a_i)$, such that $\pi_b(a_i|s_i)>0$, we define the importance-weighting factor $\rho_i = \pi(a_i|s_i)/\pi _b(a_i|s_i)$.

When $\mathcal{S}$ is large or infinite, we often use a linear approximation architecture for $V_\pi$ with parameters $\theta\in\mathbb{R}^d$ and $K$-bounded basis functions $\{\varphi_i\}_{i=1}^d$, i.e.,~$\varphi_i:\mathcal{S}\rightarrow\mathbb{R}$ and $\max_i||\varphi_i||_\infty\leq K$. We denote by $\phi(\cdot) := \big(\varphi_1(\cdot),\ldots,\varphi_d(\cdot)\big)^\top$ the feature vector and by $\mathcal{F}$ the linear function space spanned by the basis functions $\{\varphi_i\}_{i=1}^d$, i.e.,~$\mathcal{F}=\big\{f_\theta\mid\theta\in\mathbb{R}^d\;\text{and}\;f_\theta(\cdot)=\phi(\cdot)^\top\theta\big\}$. We may write the approximation of $V$ in $\mathcal{F}$ in the vector form as $\hat{v}=\Phi\theta$, where $\Phi$ is the $|\mathcal{S}|\times d$ feature matrix. 
$\xi \in \mathbb{R}^{|\mathcal{S}|}$ denotes the vector representing the stationary probability distribution over the state space $\mathcal{S}$ and depends on behavior policy $\pi_b$. We also denote by $\Xi \in \mathbb{R}^{|\mathcal{S}|\times |\mathcal{S}|}$, the diagonal matrix whose elements are $\xi(s)$.
The solution of the TD algorithm is the fixed-point solution of the following \textit{projected Bellman equation}:
\begin{align}
\hat v = \Pi {T_\pi }(\hat v),
\label{eq:tdequation}
\end{align}
where $\Pi  = \Phi {({\Phi ^ \top }\Xi \Phi )^{ - 1}}{\Phi ^ \top }\Xi $ is the weighted least-squares projection weighted by $\xi$. 

When only $n$ training samples (collected by the behavior policy $\pi_b$) are available, the sample set is denoted as $\mathcal{D}=\big\{\big(s_i,a_i,r_i=r(s_i,a_i),s'_i\big)\big\}_{i=1}^n,\;s_i\sim\xi,\;a_i\sim\pi_b(\cdot|s_i),\;s'_i\sim P(\cdot|s_i,a_i)$.  
We denote by $\delta_i(\theta) := r_i+\gamma\phi_i^{'\top}\theta-\phi_i^\top\theta$, the TD error for the $i$-th sample $(s_i,a_i,r_i,s'_i)$ and define $\Delta\phi_i=\phi_i-\gamma\phi'_i$, where $\phi_i$ (resp. $\phi_i'$) is the $i$-th feature vector w.r.t. $s_i$ (resp. $s_i'$). 
Finally, we define the covariance matrix $C$ as $C := \mathbb{E}[\phi_i\phi_i^\top] = \Phi^\top \Xi \Phi$, where the expectations are w.r.t.~$\xi$. For the $i$-th sample in the training set $\mathcal{D}$, an unbiased estimate of $C$ is $\hat{C}_i := \phi_i\phi_i^\top$.

\subsection{Oblique Projection}
\label{sec:problem}

This section introduces the oblique projection~\citep{saad:book}, the oblique projected TD methods~\citep{Scherrer:ObliqueProjection}, and then extend it to weighted oblique projected TD framework. 
The oblique projection tuple ($\Phi, X$) is
defined as follows, 
where the rows of $\Phi$ are the basis vectors for the range of the projection and the rows of $X$ are the basis vectors for the orthogonal complement of the null space of the projection.

\begin{definition}
The \textit{Oblique Projection} operator $\Pi _\Phi ^X$, 
\begin{align}
\Pi _\Phi ^X = \Phi {({X^ \top }\Phi )^{ - 1}}{X^ \top },
\end{align}
is a projection onto $span(\Phi)$ orthogonal to $span(X)$. 
\end{definition}
$\Pi^X_\Phi$ is a projection since it is idempotent: ${(\Pi _\Phi ^X)^2} = \Pi _\Phi ^X$. This projection reduces to an orthogonal projection when the basis vectors for the range are orthogonal to the null space, and is more general than the orthogonal projection. For example, the weighted least-squares projection $\Pi$ in Eq.~\eqref{eq:tdequation} can be formulated as $\Pi=\Pi^{\Xi\Phi}_\Phi$, which defines the oblique projection onto the space spanned by $\Phi$ with basis $\{\phi(s): s\in \mathcal S\}$ that is orthogonal to the space spanned by $\Xi\Phi$ with basis $\{\xi(s)\phi(s): s\in\mathcal S\}$.  


Next we introduce the oblique projected TD as a more general framework to include the TD method and the residual gradient (RG) method~\citep{Baird:ResidualAlgorithms1995}. Motivated by the extension from $\Pi$ to $\Pi^X_\Phi$, it is natural to extend the projected fixed-point equation~\eqref{eq:tdequation} with oblique projection:
\begin{align}
\hat v = \Pi _\Phi ^X{T_\pi }(\hat v).
\label{eq:obfp}
\end{align}
Figure~\ref{fig:oblique} illustrates this. Instead of minimizing the distance between $\Pi T(\hat{v})$ and $\hat{v}$, the oblique projected TD aims to minimize the distance between $\Pi^X_\Phi T(\hat{v})$ and $\hat{v}$.
\begin{figure}[h]
\centering
\includegraphics[width=.6\linewidth]{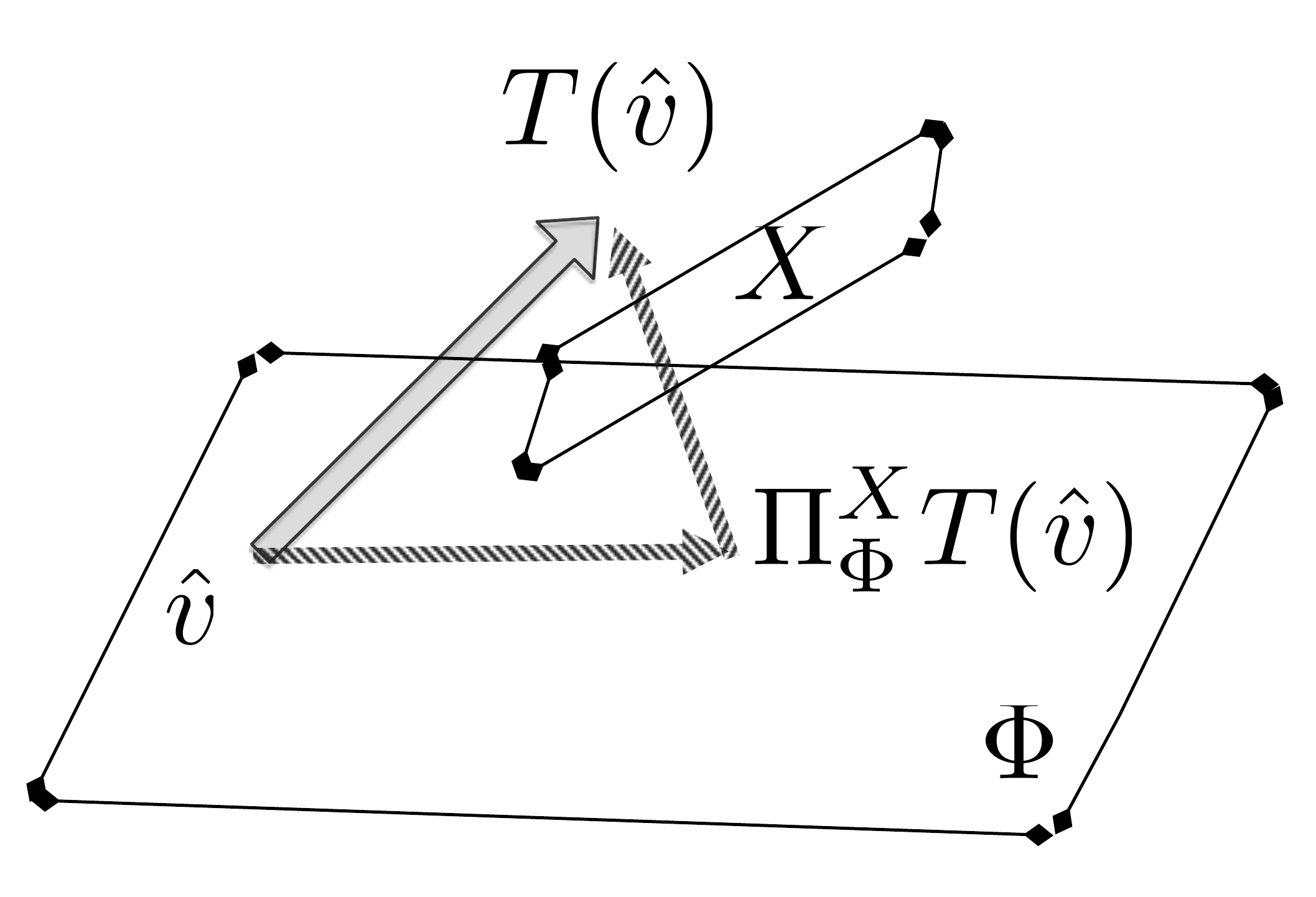}
\caption{An Illustration of Oblique Projected TD}
\label{fig:oblique}
\end{figure}

An intuitive question to ask is what the best oblique projection matrix $X$ is. Is it TD, RG,  some interpolation between them, or none of the above? To answer this, we present the following lemma, a workhorse of this paper. 
\begin{lemma}[Best projection~\citep{Scherrer:ObliqueProjection}]
\label{lem:Xbest}
Given $\Phi$, if $V_\pi$ does not lie in $span(\Phi)$, the ``best'' approximation is 
$$
v^* = \Pi V_\pi = \Phi {({\Phi ^ \top }\Xi \Phi )^{ - 1}}{\Phi ^ \top }\Xi V_\pi ,
$$ which is also the solution of the oblique projected TD equation ${v^*} = \Pi _\Phi ^{{X^*}}T{v^*}$ with 
\begin{align}
{X^*} = {({L_{\pi} ^\top })^{ - 1}}\Xi \Phi,
\label{eq:xoptimal}
\end{align}
with $L_\pi := I - \gamma {P_\pi }$.
\end{lemma}

\begin{proof}
The solution of the oblique projected fixed-point equation $\hat v = \Pi _\Phi ^XT(\hat v)$ w.r.t. the oblique projection $\Pi _\Phi ^X$ can be represented as the oblique projection $\Pi_\Phi ^{{L_{\pi} ^\top}X}$ of the true value function $V$~\citep{Scherrer:ObliqueProjection}, that is 
\begin{align}
\hat v = \Pi _\Phi ^X T_\pi(\hat v) = \Pi _\Phi ^{{L_{\pi} ^\top}X}V.
\end{align}
As $X^*$ satisfies ${v^*} = \Pi _\Phi ^{({L_{\pi } ^\top}){X^*}}V$. Let $\Pi _\Phi ^{({L_{\pi }  ^\top }){X^*}} = \Pi $, we have
$
({L_{\pi}  ^\top}){X^*} = \Xi \Phi,
$
and thus we can have Eq.~\eqref{eq:xoptimal}, which completes the proof.
\end{proof}

\subsection{Weighted Oblique Projection}
\label{sec:weightedOP}
The analytical formulation of $X^*$ is often intractable to compute in real applications. 
The major reason is that $P_\pi$ and consequently $(L^\top_\pi)^{-1}$ in $X^*$ are not known in the RL setting. To address this challenge, we introduce the weighted oblique projection matrix $Y=\Xi^{-1} X$ and derive  a stochastic approximation of $X^*$ subject to a structural simplification assumption. From the definition of $Y$, it is evident that its optimum is attained at 
$$
Y^*=\Xi^{-1} X^*=\Xi^{-1} (L^\top_\pi)^{-1}\Xi\Phi ,
$$ and the fixed point equation formulation in Eq.~\eqref{eq:obfp} becomes $\hat{v}_\theta = \Pi_\Phi^{\Xi Y} T_\pi \hat{v}_\theta$ accordingly. 
It turns out that both TD and RG solutions are weighted oblique projections with ${Y_{TD}} = \Phi$ for TD, ${Y_{RG}} = L_\pi \Phi$ for RG. 
%
Next, we discuss the necessary conditions of the existence of the fixed-point solution.
\begin{lemma}[Existence]
The solution to the \textit{weighted oblique projected Bellman equation}
\begin{align}
\hat{v}_\theta = \Pi_\Phi^{\Xi Y} T_\pi \hat{v}_\theta
\label{eq:wopTD}
\end{align} exists if ${Y^ \top }\Xi \Phi $ and ${Y^ \top }\Xi {L_\pi }\Phi$ are non-singular, and the solution is
\begin{align}
\theta  = {({Y^\top}\Xi {L_\pi }\Phi )^{ - 1}}{Y^\top}\Xi R.
\label{eq:existence}
\end{align}
\end{lemma}
\begin{proof}
\begin{align*}
  \hat{v}_\theta &= \Pi_\Phi^{\Xi Y} T_\pi \hat{v}_\theta
  \\
  \Leftrightarrow
  {Y^ \top }\Xi {\hat{v}_\theta } &= {Y^ \top }\Xi (R + \gamma {P_\pi }{\hat{v}_\theta })
  \\
  \Leftrightarrow
  Y^\top \Xi L_\pi \Phi \theta &= Y^\top \Xi R
  \\
  \Leftrightarrow
\theta  &= {({Y^\top}\Xi {L_\pi }\Phi )^{ - 1}}{Y^\top}\Xi R.
\end{align*}
The first equality holds if ${Y^ \top }\Xi \Phi$ is non-singular, and the last equality holds if ${Y^ \top }\Xi {L_\pi }\Phi$ is non-singular.
\end{proof}

Therefore, the non-singularity of ${Y^ \top }\Xi \Phi $ and ${Y^ \top }\Xi {L_\pi }\Phi$ guarantees the existence of $\Pi^{\Xi Y}_\Phi$ and $\theta$ as in Eq.~\eqref{eq:existence}. The extension from oblique projection to weighted oblique projection, though technically trivial, enables the design of stochastic approximation-based algorithms.

\section{Algorithm Design}
\label{sec:alg}

This section presents the design of the stable and efficient algorithm. We first present the motivation to use the oblique projection. Then, a computationally efficient criterion is proposed to overcome the computational intractability to compute $X^*$. Based on this criterion, an algorithm is proposed based on a diagonal approximation and is also extended to the multi-step learning setting with eligibility trace.

\subsection{Motivation}
 
This paper aims at achieving off-policy stability for TD learning in off-policy settings.
It is well-known that the TD method with linear function approximation has instability issues in off-policy learning settings~\citep{sutton-barto:book,tsitsiklis-roy:tdfun}, which is largely due to the limitation of the projected fixed-point formulation in Eq.~\eqref{eq:tdequation}.
TD solution, as a projected fixed-point formulation, is highly sensitive to the degree of ``off-policyness'', i.e., the difference between the behavior policy $\pi_b$ and the target policy $\pi$. 
On the other hand, $\Pi V$, being the ``best'' approximation (by the representation space $span(\Phi)$) of the true value function $V$, is always unique and stable, yet is difficult to compute in reinforcement learning settings. 
It is therefore desirable to propose a novel fixed-point formulation whose solution is close to the best approximation $\Pi V$ to enable off-policy stability. One possible way to achieve this is to use the weighted oblique projection operator $\Pi_\Phi^{\Xi Y}$ and change the vanilla projected Bellman formulation in Eq.~\eqref{eq:tdequation} to the weighted oblique projected Bellman formulation in Eq.~\eqref{eq:wopTD}.
Closeness to $\Pi V$ implies that the solution is less sensitive to the ``off-policyness'' than the TD solution. In a nutshell, this paper aims at proposing a weighted oblique projected TD framework in Eq.~\eqref{eq:wopTD} to achieve off-policy stability via forcing proximity to the ``best'' approximation $\Pi V$, with stochastic approximation methods.

\subsection{Approximation Criteria}
We first introduce a simple but important property of the optimal projection matrix  $X^*$. We denote $\Lambda := L_\pi \Phi$. As $X^* = {(L_{\pi} ^\top)^{-1} }\Xi \Phi $, we have
\begin{align}
{\Lambda ^ \top }{X^*} = {\Phi ^ \top }({L_{\pi } ^\top}){({L_{\pi } ^\top})^{ - 1}}\Xi \Phi  = {\Phi ^ \top }\Xi \Phi  = C.
\label{eq:p1}
\end{align} 

Motivated by this, Proposition~\ref{pro:fundamental} is presented to formulate the cornerstone of this paper. 
\begin{proposition}
If the weighted oblique projection $Y$ satisfies ${Y^ \top }\Xi L_\pi \Phi  = C$,
and if $\Phi$ has full row rank (${{rank}}(\Phi) = |\mathcal{S}|$), then we have $Y = Y^*$. 
\label{pro:fundamental}
\end{proposition}
\begin{proof}
\begin{align}
{Y^ \top }\Xi L_\pi \Phi  &= C \\
\Leftrightarrow
{Y^ \top }\Xi L_\pi \Phi  &= {\Phi ^ \top }\Xi \Phi  \\
\Leftrightarrow
{L_\pi^\top }\Xi  Y &= \Xi \Phi \text{ (as ${{rank}}(\Phi) = |\mathcal{S}|$)}
\\
\Leftrightarrow
\Xi  Y &= {({L_\pi ^\top })^{ - 1}}\Xi \Phi  = {X^*}
\\
\Leftrightarrow
Y^* &= Y
\end{align}
\end{proof}

Although the rank condition is restrictive in real applications, it still offers helpful directions to approximate $X^*$ in a computationally efficient way.

\subsection{SETD Algorithm Design}

In this paper, we investigate a special type of weighted oblique projection: $Y$ can be decomposed into the product of a $|\mathcal{S}| \times |\mathcal{S}|$ diagonal matrix $\Omega$ and $\Phi$ such that $Y=\Omega\Phi$.
The optimal $\Omega$, termed as $\Omega^*$, can be obtained via
\begin{align}
\Omega^*  = \arg \mathop {\min }\limits_\Omega ||\Xi \Omega \Phi  - X^{*}||_F,
\label{eq:w_opt}
\end{align}
which is impossible to compute since $X^*$ is unknown.
With Eq.~\eqref{eq:p1}, an approximation of $\Omega^*$, termed as $\Omega_S$, can be computed as:
\begin{align}
\Omega_S  = \arg \mathop {\min }\limits_\Omega  ||{\Lambda ^ \top }\Xi \Omega \Phi  - C|{|_F}.
\label{eq:batch}
\end{align}
The optimization problem reduces to a matrix regression problem, with $||\cdot||_F$ the Frobenius norm.
With the sample-based estimation of ${{ \Lambda}^ \top }\Xi \Omega_S \Phi $ and $C$ matrices, i.e, 
\begin{align}
{{ \Lambda}^ \top }\Xi \Omega_S \Phi & \leftarrow \frac{1}{n}\sum\limits_{i = 1}^n {{\omega _i}{(\phi_i - \gamma \phi_i')}}{\phi_i ^ \top }  \\
C & \leftarrow \frac{1}{n} \sum\limits_{i = 1}^n {{\phi _i}{\phi_i ^\top}} ,
\end{align}
with $\omega _i$ the $i^\text{th}$ diagonal element of $\Omega_S$, the problem is formulated as
$$
\mathop {\min }\limits_{\omega_i}  \frac{1}{n}||\sum\limits_{i = 1}^n {({\omega _i}{\Delta \phi _i} \phi _i^ \top  - {\phi _i}\phi _i^ \top )||_F} .
$$
So Eq.~\eqref{eq:batch} can be approximated as 
\begin{align}
\Omega_S  
&\approx \arg \mathop {\min }\limits_{\omega_i}  \frac{1}{n}||\sum\limits_{i = 1}^n {({\omega _i}{\Delta \phi _i} \phi _i^ \top  - {\phi _i}\phi _i^ \top )||_F} .
\label{eq:batch2}
\end{align}

Now, we propose a relaxed method to address this problem based on two observations. First, it is desirable that $\forall i$,   $\Omega_S ({s_i},{s_i})$ be positive. This is intuitive.
Secondly, instead of solving the above objective function, a relaxed sample-separable objective function $\Omega_S$ using the triangle inequality can be formulated as follows by denoting ${\omega _i}: = \Omega_S ({s_i},{s_i})$,
\begin{align}
\forall i,  {\omega _i} = \arg \mathop {\min }\limits_\omega  || \omega \Delta {\phi _{i}}\phi _i^ \top  - {\phi _{i}}\phi _i^ \top||_F, \quad\rm{s.t.}\quad \omega _i \ge 0 .
\label{eq:sto2}
\end{align}
The closed-form solution of $\omega_i$ is 
\begin{align}
\omega_i = \max\bigg(\frac{{\Delta \phi _{i}^ \top {\phi _{i}}}}{{||\Delta {\phi _{i}}|{|^2}}},0 \bigg),
\label{eq:omega-i}
\end{align}
where $||\cdot||$ is the $\ell_2$-norm of a vector. 

Here we show the detailed deduction. To obtain Eq.~\eqref{eq:omega-i}, we first introduce the following lemmas to compute the singular value of rank-$1$ matrices.
We first introduce Lemma~\ref{lem:rank1} without proof, which is instrumental in the theoretical proof. 

\begin{lemma}
A rank-$1$ real-valued square matrix $G=pq^\top$ where $p,q$ are vectors of the same length, the eigenvalues of $G$ are 
\begin{align}
\lambda (G) = \{ {p^\top}q,0,0,0, \cdots \},
\end{align}
i.e., $G$ has only one nonzero eigenvalue ${p^\top}q$, and all other eigenvalues are $0$, and 
${{Tr}}(G) = {p^\top}q$, where ${{Tr}}(\cdot)$ is the trace of a matrix.
\label{lem:rank1}
\end{lemma}

Then we introduce Lemma~\ref{lem:rank1unsquare}.

\begin{lemma}
A rank-$1$ real matrix (not necessarily square) $M=uv^\top$ has only one nonzero singular value $\sigma_{\max} (M)  = ||u||_2 \cdot ||v||_2$, where $||\cdot||_2$ is the $\ell_2$-norm of a vector, and the Frobenius norm and the trace norm of $M$ are identical, i.e.,
\begin{align}
||M||{_*} = ||M||{_F} = \sigma_{\max} (M)  = ||u||_2 \cdot ||v||_2 .
\label{eq:equivalence}
\end{align}
\label{lem:rank1unsquare}
\end{lemma}
\begin{proof}
We use $M^H$ to represent the conjugate transpose of the $M$ matrix, and $\lambda (\cdot)$ to represent the eigenvalues of a square matrix.
Then we have
\begin{align}
\lambda ({M^{{H}}}M) 
= \lambda (v{u^ \top }u{v^ \top }) 
= ({u^ \top }u)\lambda (v{v^ \top }) 
\end{align}
From Lemma~\ref{lem:rank1}, we know that $\lambda (vv^\top)$ are $\{ {v^ \top }v,0,0, \cdots \}$,
and thus $M$ has only one nonzero singular value ${\sigma _{\max }}(M)$:
\begin{align}
{\sigma_{\max } }(M) 
&= \sqrt {\lambda ({M^{{H}}}M)} 
= \sqrt {{\lambda }(v{u^ \top }u{v^ \top })}  \\
&= \sqrt {({u^ \top }u){\lambda}(v{v^ \top })} 
= \sqrt {({u^ \top }u)({v^ \top }v)}   \\
&= ||u||_2 \cdot ||v||_2,
\end{align}
and all other singular values of $M$ are $0$. Thus $||M||_* = ||M||_F = ||u||_2 \cdot ||v||_2$, which completes the proof.
\end{proof}

Based on Lemma~\ref{lem:rank1unsquare}, we now show the derivation of Eq.~\eqref{eq:omega-i}.
To tackle the following trace norm minimization formulation,
\begin{align}
{\omega _i} = \arg \mathop {\min }\limits_\omega  ||\omega {\Delta \phi _i} \phi _i^ \top  - {\phi _i}\phi _i^ \top ||{_*},
\label{eq:tracemin}
\end{align}
we need to use the structure of the rank-$1$ matrices. We have
\begin{align}
\omega {\Delta \phi _i} \phi _i^ \top  - {\phi _i}\phi _i^ \top  = {(\omega \Delta {\phi _i} - {\phi _i}) {\phi _i}^ \top },
\end{align}
we denote $q_i(\omega) := {(\omega \Delta {\phi _i} - {\phi _i}) }$, and thus we have
\begin{align}
\nonumber
\omega_i &=
\arg \mathop {\min }\limits_\omega  ||{q_i}(\omega ){\phi _i ^\top}||{_*}\\
\nonumber 
&= \arg \mathop {\min }\limits_\omega  ||{\phi _i }||_2 \cdot ||q_i(\omega)||_2 \\
&= \arg \mathop {\min }\limits_\omega  ||q_i(\omega)||_2 .
\label{eq:traceform}
\end{align}
The second equality above comes from Eq.~\eqref{eq:equivalence}, and the third equality from the fact that $||{\phi _i}||_2$ does not depend on $\omega$.

On the other hand, using $||\cdot||^2_F$ instead of trace norm in Eq.~\eqref{eq:tracemin}, we have
\begin{align}
{\omega _i} &=\arg \mathop {\min }\limits_\omega  ||{q_i}(\omega ){\phi _i ^\top}||_F^2,
\label{eq:fromin}
\\
\text{with }
||{q_i}(\omega ){\phi _i ^\top}||_F^2 
\nonumber 
&= {{Tr}}({\phi _i}{q^ \top _i}(\omega ) {q_i}(\omega ){\phi ^ \top _i })\\ 
\nonumber 
&= ({\phi ^ \top _i}{\phi _i}){{Tr}}({q_i}(\omega ){q^ \top _i}(\omega ))\\
\nonumber 
&= ({\phi ^ \top _i}{\phi _i})({q^ \top _i}(\omega ){q_i}(\omega ))\\
&=  ||{\phi _i}||_2^2||{q_i}(\omega )||_2^2.
\end{align}
The first equality comes from the fact that 
$
||M||_F^2 = {{Tr}}({M^H}M).
$
The third equality comes from Lemma~\ref{lem:rank1}.
Then we can see that Eq.~\eqref{eq:fromin} is equivalent to Eq.~\eqref{eq:traceform}, as verified by Lemma~\ref{lem:rank1unsquare}. 
So both trace norm and Frobenius norm minimizations are equivalent to
\begin{align}
\omega_i = \arg \mathop {\min }\limits_\omega  || \omega \Delta {\phi _i} - {\phi _i}||^2_2
\label{eq:tr}
\end{align}
By zeroing the gradient of the right hand-side of Eq.~\eqref{eq:tr}, we will have Eq.~\eqref{eq:omega-i} as the final result, which is also the vector projection weight of $\phi_i$ projected onto $\Delta\phi_i$.

For the on-policy case, the update rule is now ready as ${\theta _{i + 1}} = {\theta _{i}} + \alpha_i {\omega _i}{\delta _i}{\phi _i}$, where $\alpha_i \in (0,1]$ is the stepsize. 
For the off-policy case, importance weights ${\rho _i} = \frac{{\pi ({a_i}|{s_i})}}{{{\pi _b}({a_i}|{s_i})}}$ is used to enable the algorithm to take into consideration the discrepancies between the behavior policy $\pi$ and the target policy $\pi_b$ by properly weighing the observation, which is a standard way in off-policy learning~\citep{dann2014tdsurvey,geist2014off:trace:jmlr}:
%
\begin{align}
&\mathbb{E}_{\pi} \big[\delta_i \omega_i \phi_i \big] \nonumber \\ 
=& \sum_{s_{i+1}} \sum_{a_i} \sum_{s_i} P(s_i, a_i, s_{i+1})\delta_i \omega_i \phi_i \nonumber \\ 
\approx & \sum_{s_{i+1}} \sum_{a_i} \sum_{s_i} P(s_{i+1}|s_i, a_i) \pi(a_i|s_i) \xi(s_i) \delta_i \omega_i \phi_i \nonumber \\ 
= & \sum_{s_{i+1}} \sum_{a_i} \sum_{s_i} P(s_{i+1}|s_i, a_i) \pi_b(a_i|s_i) \xi(s_i) \frac{\pi(a_i|s_i)}{\pi_b(a_i|s_i)}\delta_i \omega_i \phi_i \nonumber \\
= & \mathbb{E}_{\pi_b} \big[\rho_i \delta_i \omega_i \phi_i \big].
\end{align}
The update rule is thus defined as,
\begin{align}
{\theta _{i + 1}} = {\theta _{i}} + \alpha_i \rho _i{\omega _i}{\delta _i}{\phi _i},
\label{eq:theta-diag}
\end{align}
where $\alpha_i \in (0,1]$ is the stepsize.
The resulting \textit{Stable and Efficient TD Algorithm} (\textbf{SETD}) is in Algorithm~\ref{alg:o2td}. 
\begin{algorithm}
\caption{Stable and Efficient TD Algorithm (SETD)}
\label{alg:o2td}
\begin{algorithmic}[1]
\STATE INPUT: Sample set $\{ {\phi _i},{r_i},{\phi _i}^\prime \} _{i = 1}^n$
\FOR {$i=1,\ldots,n$}
\STATE Compute $\phi_i, \Delta \phi_i$, $\delta_i=r_i+\gamma\phi_i^{'\top}\theta_i-\phi_i^\top\theta_i$.
\STATE Compute $\omega_i$ according to Eq.~\eqref{eq:omega-i}.
\STATE Compute $\theta_{i+1}$ according to Eq.~\eqref{eq:theta-diag}.
\ENDFOR
\end{algorithmic}
\end{algorithm}
The computational cost per step is $O(d)$, as can be seen from the computation of Eq.~\eqref{eq:omega-i} and \eqref{eq:theta-diag}.
\subsection{Extension to Eligibility Traces}
Here we extend SETD to eligibility traces.
First, we introduce the general $\lambda$-return with bootstrapping and discounting based on the importance-weighting factor by using the TD forward view:
\begin{align}
    G_t^{\lambda \rho}(V) = \rho_t \Big(r_{t+1} + \gamma \big[(1- \lambda)V(S_{t+1}) + \lambda G_{t+1}^{\lambda \rho} \big] \Big),
\end{align}
and define the value function at $s$ for a given policy $\pi$:  
\begin{align}
V_{\pi}(s) = \mathbb{E}\big[G_t^{\lambda \rho}(V_{\pi})|S_t =s, \pi \big] = {T_\pi^{\lambda \rho}} V_{\pi}(s) ,
\end{align}
where $\lambda \in [0,1]$ is the bootstrapping parameter and $T_\pi^{\lambda \rho}$ is the $\lambda$-weighted Bellman operator for policy $\pi$. Using linear function approximation, we get the TD equation
\begin{align}
b - A \theta &= \Phi ^\top \Xi \Omega_S R - \Phi ^\top \Xi \Omega_S \Lambda \theta \nonumber \\
& = \Phi ^\top \Xi \Omega_S ({T_\pi^{\lambda \rho}}V_{\theta} -  V_{\theta}).
\end{align}
Define $\delta_t^{\lambda \rho}(\theta) := G_t^{\lambda \rho}(\theta) - \phi_t^\top \theta$ and 
\begin{align}
P_{\xi}^{\pi}\delta_t^{\lambda \rho}(\theta) \omega_t \phi_t := \sum_s \xi(s) \mathbb{E} \big[\delta_t^{\lambda \rho}(\theta)| S_t=s,\pi \big]\omega_t \phi_t ,
\end{align}
where $P_{\xi}^{\pi}$ is an operator.
We also have:
\begin{align}
\mathbb{E} \big[\delta_t^{\lambda \rho}(\theta)| S_t=s,\pi \big] &= T_{\pi }^{\lambda \rho}V_{\theta}(s) -  V_{\theta}(s) \\
P_{\xi}^{\pi}\delta_t^{\lambda \rho}(\theta) \omega_t \phi_t & = \sum_s \xi(s) \mathbb{E} \big[\delta_t^{\lambda \rho}(\theta)| S_t=s, \pi \big]\omega_t \phi_t  \nonumber \\
&=\sum_s \xi(s) \big[T_{\pi} ^{\lambda \rho} V_{\theta}(s) -  V_{\theta}(s) \big]\omega_t \phi_t  \nonumber \\
& = \Phi ^\top \Xi \Omega_S (T_\pi^{\lambda \rho} V_{\theta} -  V_{\theta}).
\end{align}
Therefore, we have $P_{\xi}^{\pi}\delta_t^{\lambda \rho}(\theta) \omega_t \phi_t = \mathbb{E} \big[\delta_t^{\lambda \rho}(\theta)\omega_t \phi_t \big]$.

Consider the following identities:
\begin{align}
\delta_t^{\lambda \rho}(\theta) 
= &G_t^{\lambda \rho}(\theta) - \phi_t^\top \theta  \nonumber \\
 = & \rho_t \Big(r_{t+1} + \gamma \big[(1- \lambda)\phi_{t+1}^\top \theta + \lambda G_{t+1}^{\lambda \rho} \big] \Big) - \phi_t^\top \theta  \nonumber \\ 
 = & \rho_t \big(r_{t+1} + \gamma \phi_{t+1}^\top \theta -  \phi_t^\top \theta + \phi_t^\top \theta \big) \nonumber \\
 & - \rho_t \gamma \lambda \phi_{t+1}^\top \theta + \rho_t \gamma \lambda  G_{t+1}^{\lambda \rho} - \phi_t^\top \theta  \nonumber \\ 
 = &\rho_t \big(r_{t+1} + \gamma \phi_{t+1}^\top \theta -  \phi_t^\top \theta \big) \nonumber \\ 
&+ \rho_t \gamma \lambda \big(G_{t+1}^{\lambda \rho}-\phi_{t+1}^\top \theta \big)  +  \rho_t \phi_t^\top \theta- \phi_t^\top \theta  \nonumber \\ 
=&  \rho_t \delta_t(\theta) +  \rho_t \gamma \lambda \delta_{t+1}^{\lambda \rho}(\theta) +(\rho_t -1)\phi_t^\top \theta,
\end{align}
and
\begin{align}
& \mathbb{E} \big[(\rho_t -1)\phi_t^\top \theta \big]  \nonumber \\
&= \sum_s \xi(s)  \sum_a \pi_b(a|s)(\rho_t -1)\phi_t^\top \theta \nonumber \\ 
& = \sum_s \xi(s) \Big(\sum_a \pi(a|s) -\sum_a \pi_b(a|s) \Big) \phi_t^\top \theta = 0.
\end{align}
Hence, we can get:
\begin{align}
b - A \theta  
=& \Phi ^\top \Xi \Omega_S (T_{\pi}^{\lambda \rho}V_{\theta} -  V_{\theta}) 
=P_{\xi}^{\pi}\delta_t^{\lambda \rho}(\theta) \omega_t \phi_t \nonumber \\
=& \mathbb{E} \big[\delta_t^{\lambda \rho}(\theta)\omega_t \phi_t \big] \nonumber \\ 
= & \mathbb{E} \Big[\rho_t \delta_t(\theta)\omega_t \phi_t +  \rho_t \gamma \lambda \delta_{t+1}^{\lambda \rho}(\theta)\omega_t \phi_t  \nonumber \\ 
&+(\rho_t -1)(\phi_t^\top \theta) \omega_t \phi_t \Big] \nonumber \\ 
=& \mathbb{E} \Big[\rho_t \delta_t(\theta)\omega_t \phi_t +  \rho_t \gamma \lambda \delta_{t+1}^{\lambda \rho}(\theta)\omega_t \phi_t \Big] \nonumber \\ 
=& \mathbb{E} \Big[\rho_t \delta_t(\theta)\omega_t \phi_t +  \rho_{t-1} \gamma \lambda \delta_{t}^{\lambda \rho}(\theta)\omega_{t-1} \phi_{t-1} \Big] \nonumber \\ 
=& \mathbb{E} \Big[\rho_t \delta_t(\theta)\omega_t \phi_t +  \rho_{t-1} \gamma \lambda \big(\rho_t \delta_t(\theta)  \nonumber \\ 
&+  \rho_t \gamma \lambda \delta_{t+1}^{\lambda \rho}(\theta) +(\rho_t -1)\phi_t^\top \theta \big)\omega_{t-1} \phi_{t-1} \Big] \nonumber \\ 
=& \mathbb{E} \Big[\rho_t \delta_t(\theta) \big(\omega_t \phi_t +\rho_{t-1} \gamma \lambda \omega_{t-1} \phi_{t-1}  \nonumber \\ 
& + \rho_{t-2} \rho_{t-1} \gamma^2 \lambda^2 \omega_{t-2} \phi_{t-2} + \cdots \big)  \Big] \nonumber \\ 
=& \mathbb{E} \Big[ \delta_t(\theta) e_t \Big],
\end{align}
where the eligibility trace vector is defined as $e_t = \rho_t (\omega_t \phi_t + \gamma \lambda e_{t-1})$.
So we can now specify our final new algorithm, SETD($\lambda$), by the following steps, for $t > 0$:
\begin{align}
w_t &= \max\bigg(\frac{\Delta \phi_t ^\top \phi_t}{||\Delta \phi_t ||^2}, 0 \bigg),
\label{eq:w_trace}
\end{align}
\begin{align}
e_{t} &= \rho_t(\lambda \gamma e_{t-1} + w_t \phi_t), \\
\theta_{t+1} &= \theta_t + \alpha \delta_t e_t ,
\end{align}
where $ e_{0} = \vec{0}$. It is shown in Algorithm~\ref{alg:o2tdtrace}.

\begin{algorithm}
\caption{SETD($\lambda$)}
\label{alg:o2tdtrace}
\begin{algorithmic}[1]
\STATE INPUT: Sample set $\{ {\phi _t},{r_t},{\phi _t}^\prime \} _{t = 1}^n$
\STATE Initialize $e_{0} = {\vec 0}$.
\FOR {$t=1,\ldots,n$}
\STATE Compute $\phi_t, \Delta \phi_t$, $\delta_t=r_t+\gamma\phi_t^{'\top}\theta_t-\phi_t^\top\theta_t$.
\STATE Compute $\omega_t$ according to Eq.~\eqref{eq:w_trace}.
\STATE Compute $e_{t} = \rho_t(\lambda \gamma e_{t-1} + \omega_t \phi_t)$.
\STATE Compute $\theta_{t+1} = \theta_t + \alpha_t \delta_t e_t$.
\ENDFOR
\end{algorithmic}
\end{algorithm}



\section{Related Work}
\label{sec:related}

This section presents the related work, which is primarily the emphatic TD method (ETD) algorithm~\citep{etd:sutton2015} and the Retrace($\lambda$) algorithm~\citep{saferl:munos2016safe}.

The Retrace($\lambda$) algorithm shares the same motivation with the SETD algorithm, i.e, off-policy stability, and on-policy efficiency. To this end, it uses a capped importance ratio technique, which is shown to be superior to the conventional importance sampling method and Tree backup method~\citep{precup:tb:2000}. This research direction is complementary to our research and has the potential to combine with the SETD method, which is left for future research due to space limitations.

To the best of our knowledge, the closest work to ours is the ETD~\citep{etd:sutton2015}. ETD has indeed a  weighted oblique projection structure similar to SETD, with  ${Y_E} = {\Omega _{E}}\Phi$. The $i^\text{th}$ diagonal element ${\Omega _{E}}(i,i)$ is computed as 
\begin{align}
{{\Omega _{E}}(0,0)} &= 1; \\
\nonumber
{{\Omega _{E}}(i,i)} &= 1 + \gamma {\rho _{i - 1}}{{{\Omega _{E}}(i-1,i-1)}}, \quad i>0.
\end{align}
and the ETD algorithm update law is
\begin{align}
\theta_{i+1} = \theta_i + \alpha_i {{\Omega _E}(i,i)} \rho_i \delta_i \phi_i.
\label{eq:theta_etd}
\end{align}

A more detailed explanation of the ETD algorithm from the oblique projection perspective is shown as follows. Similar to SETD, ETD also assumes that the weighted oblique projection $Y_E$  can be approximated by the product of a diagonal matrix (termed as $\Omega_E$)  and $\Phi$, i.e., $Y_E = \Omega_E\Phi$. Then a different technique is used based on the power series expansion, i.e., 
\begin{align}
{({L_\pi })^{ - 1}} = {(I - {\gamma P_\pi })^{ - 1}} = \sum\limits_{k = 0}^\infty  {{{({\gamma P_\pi })}^k}}. 
\end{align}
Then the power series expansion is used to compute $\Xi \Omega$ as a whole.  
Since the optimal oblique projection matrix is ${X^*} = {({L_{\pi } ^\top})^{ - 1}}\Xi \Phi $, it is evident that $\hat X_E = \Xi \Omega_E \Phi$ should be as close as possible to $X^*$, especially the diagonal elements. The diagonal elements of $\hat X_E$ are represented as a (column) vector $f$.
One conjecture is that for the diagonal matrix of $\hat{X_E}$, it is desired that
$
f = {({L_{\pi }^ \top})^{ - 1}}{\xi } .
$
By using the power series expansion, $f$ can be expanded as
\begin{align}
f &= {({L_{\pi }^ \top })^{ - 1}}{\xi } = (\sum\limits_{k = 0}^\infty  {{{(\gamma {P_{\pi }^\top})}^k}} ){\xi }\\
 &= (I + \gamma {P_{\pi }^\top} + {(\gamma {P_{\pi }^\top})^2} +  \cdots + {(\gamma {P_{\pi }^\top})^k} + \cdots){\xi }.
\end{align}
Readers familiar with the emphatic TD learning algorithm know that this is actually identical to Eq.~(13) in~\cite{etd:sutton2015}, where a scalar follow-on trace is computed as\footnote{We use subscript $\bullet_t$ to denote sequential samples, and subscription $\bullet_i$ to denote samples that are randomly sampled with replacement.}
\begin{align}
{{\Omega _{E}}(0,0)} &= 1; \nonumber \\
{{\Omega _{E}}(t,t)} &= 1 + \gamma {\rho _{t - 1}}{{{\Omega _{E}}(t-1,t-1)}}, \quad t>0,
\label{eq:followon}
\end{align}
with ${\Omega_E}(t,t)$ denoting the $t^\text{th}$ diagonal element of $\Omega _E$.
It turns out that 
\begin{align}
{f_i} = {\xi }(i)\mathop {\lim }\limits_{t \to \infty } \mathbb{E}[{\Omega _{E}(i,i)}|{S_t} = {s_i}],
\end{align}
which leads to the standard emphatic TD($0$) algorithm,
\begin{align}
\theta_{t+1} = \theta_t + \alpha_t \Omega _E(t,t) \rho_t \delta_t \phi_t.
\label{eq:theta_etd0}
\end{align}

Previous works~\citep{yu2015convergence,mahmood2017incremental} also associated  ETD  with oblique projection. This sheds a helpful light on understanding the family of the emphatic TD learning algorithms. However, the ETD algorithm requires the sequential sampling condition, i.e., $s'_i = s_{i+1}, \forall i>0$, which is not suitable for a set of samples collected from many episodes. This restriction is alleviated for the SETD algorithm.

\begin{figure}[h]
\centering
\includegraphics[width=.45\linewidth]{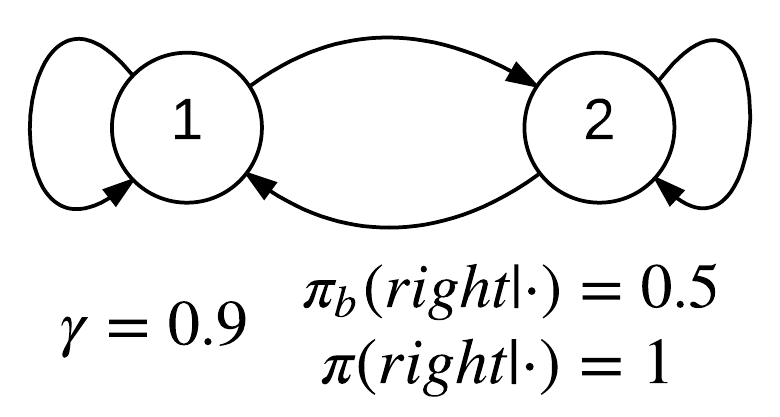}
\caption{Two-state MDP.}
\label{fig:two-states}
\end{figure}

We compare the two algorithms on the 2-state MDP of~\citep{etd:sutton2015}. 
As shown in Figure~\ref{fig:two-states}, this environment has two actions, left and right, which take the process to the left or right states. The single feature is $[1,2]^\top$ in the two states, and the discount factor $\gamma=0.9$. The behavior policy $\pi_{b}$ is to go left and right with equal probability from both states, while the target policy $\pi$ is to go right in both states.

Since $P_{\pi} = \begin{bmatrix} 
    0 & 1  \\
    0  & 1  
    \end{bmatrix}$,  we have
\begin{align}
L_{\pi} = I - \gamma P_{\pi} = \begin{bmatrix} 
    1 & -0.9  \\
    0  & 0.1  
    \end{bmatrix},
\end{align} 
and 
\begin{align}
X^{*} = (L_{\pi}^\top)^{-1} \Xi \Phi = \begin{bmatrix} 
    0.5  \\
    14.5  
    \end{bmatrix},
\end{align}
with  $X^{*}$ is the optimal oblique projection matrix and $\Xi = 0.5I$.

Next, we compute the oblique projection matrices, $X_{\Omega_E}$ and $X_{\Omega_S}$, for SETD and ETD, respectively. 

According to the definition of matrix $\Omega_E$ in ETD algorithm by \citep{etd:sutton2015}, let's use a (column) vector $f_{ETD}$  to represent the diagonal elements of matrix $\Omega_E$, it is calculated as:
\begin{align}
f_{ETD} = (I - \gamma P_{\pi}^\top)^{-1}\xi = \xi + \gamma P_{\pi}^\top \xi + (\gamma P_{\pi}^\top)^2 \xi + \cdots .
\end{align}
So $f_{ETD}(1) = 0.5$ and $f_{ETD}(2) = 0.5 + 0.9 + (0.9)^2 + (0.9)^3 + \cdots = 9.5$. Therefore we can get 
\begin{align}
\Omega_E = \begin{bmatrix} 
        0.5 & 0  \\
        0  & 0.95  
    \end{bmatrix}.
\label{eq:Omega_E}
\end{align}
Then we have 
\begin{align}
X_{\Omega_E} = \Xi \Omega_E \Phi = \begin{bmatrix} 
        0.25  \\
        9.5  
    \end{bmatrix}.
\label{eq:X_omega_E}
\end{align}

For SETD algorithm, each diagonal entry of $\Omega_S$ can be computed according to Eq.~\eqref{eq:omega-i}, which is 
\begin{align}
\Omega_S = \begin{bmatrix} 
        0 & 0  \\
        0  & 10  
    \end{bmatrix}.
\label{eq:Omega_S}
\end{align}
Then we have
\begin{align}
X_{\Omega_S} = \Xi \Omega_S \Phi = \begin{bmatrix} 
        0  \\
        10  
    \end{bmatrix}.
\label{eq:X_omega_S}
\end{align}
%
%
%
Here is a summary of comparing the SETD and ETD on the 2-state MDP domain, as shown in Table~\ref{tab:2state}. It should be noted that though TD will diverge on this domain, the solution to TD is the upper bound. As mentioned in Section~\ref{sec:weightedOP}, the weighted oblique projection structure for TD is $Y_{TD} = \Phi$. This is equivalent to 
$
Y_{TD} = \Omega_{TD} \Phi,
$
where 
\begin{align}
\Omega_{TD} = I.
\label{eq:Omega_TD}
\end{align}
Then we have 
\begin{align}
X_{TD} = \Xi \Omega_{TD} \Phi = \begin{bmatrix} 
        0.5  \\
        1  
    \end{bmatrix}
\label{eq:X_omega_TD}
\end{align} for this 2-state MDP domain.
Table~\ref{tab:2state} shows the comparison of SETD and ETD in the diagonalized approximation of $X^*$. For TD, 
\begin{align}
||{\Lambda ^ \top }\Xi \Omega \Phi  - C||{^2_F}  = ||\Lambda ^ \top \Xi \Phi - \Phi^\top \Xi \Phi||{^2_F} 
 = ||\gamma (\Phi'^\top\Xi\Phi)||^2_F.
\end{align}
According to the result, it can be seen that SETD performs better than ETD in $X^*$ estimation.

\begin{table}[h]
\caption{Comparison on the $2$-state MDP} 
\label{tab:2state}
\centering
\begin{tabular}{|c|c|c|c|c|}
\hline
{\bf Algorithm} & $\Omega$ & $||{\Lambda ^ \top }\Xi \Omega \Phi  - C||{^2_F}$ & X &{ $||X - {X^*}||_2$ } \\
\hline 
SETD   & Eq.\eqref{eq:Omega_S}  & \textbf{0.25} & Eq.\eqref{eq:X_omega_S} & \textbf{4.5277} \\
\hline
ETD    & Eq.\eqref{eq:Omega_E} & 0.64 & Eq.\eqref{eq:X_omega_E} & 5.0062 \\
\hline 
TD    & Eq.\eqref{eq:Omega_TD} & 7.29 & Eq.\eqref{eq:X_omega_TD} & 13.5 \\
\hline
\end{tabular}
\end{table}

\section{Experimental Study}
\label{sec:experimental}
This section evaluates the effectiveness of the proposed algorithms, comparing SETD with TD, GTD2, TDC (TD with gradient correction term), and ETD for on-policy learning and off-policy learning, respectively.
It should be mentioned that since the major focus of this paper is value function approximation, comparisons on control learning performance are not reported here. 
We use $\alpha_{{TD}}$, $\alpha_{{ETD}}$, $\alpha_{{SETD}}$, $\alpha_{{GTD2}}$, $\mu_{{GTD2}}$ ($\beta_{{GTD2}} = \alpha_{{GTD2}}*\mu_{{GTD2}}$) and $\alpha_{{TDC}}$, $\mu_{{TDC}}$ ($\beta_{{TDC}} = \alpha_{{TDC}}*\mu_{{TDC}}$) to denote the stepsizes for TD, ETD, SETD, GTD2, and TDC respectively. 
In order to focus on the algorithm itself and make the comparison fair, which is similar to~\citep{pgtd:adam2016,dann2014tdsurvey}, only constant stepsize is considered in this paper. All stepsizes are chosen via a range of parameters similar to ~\citep{dann2014tdsurvey} that are based on grid search method, as shown in Table~\ref{tab:gridsearch}. 

\begin{table}[h]
\caption{Considered values for grid-search.} 
\label{tab:gridsearch}
\centering
\begin{tabular}{ |c|c|} 
\hline
{\bf Parameter} & {\bf Evaluated Values} \\
\hline 
\multirow{4}{*}{$\alpha $}   & $1*10^{-7}, 1*10^{-6},1*10^{-5}, 1*10^{-4}, 0.001, 0.01, 0.1,$   \\
 & $3*10^{-6}, 5*10^{-6}, 7*10^{-6}, 9*10^{-6},2*10^{-6},$ \\
 & $ 2.5*10^{-6}, 2*10^{-4}, 4*10^{-4}, 6*10^{-4}, 8*10^{-4},$ \\
 & $0.002,...,0.009, 0.02,...,0.09, 0.2, 0.3, 0.4, 0.5, 0.6$ \\
\hline
$\mu$    & $1*10^{-4}, 1*10^{-3}, 0.01, 0.1, 1, 4, 8, 16, 0.005, 0.05, 0.5$ \\
\hline 
$\lambda$    & $0.4, 0.8$  \\
\hline
\end{tabular}
\end{table}

Two metrics, Root Mean-Squares Error (RMSE) and Root Mean-Squares Projected Bellman Error (RMSPBE)~\citep{tdc:2009}, are used as the performance measure:
\begin{align}
{\rm RMSE}  = \sqrt{||\hat{v} - V ||_{\Xi}^{2}}, \quad
{\rm RMSPBE}  = \sqrt{||\hat{v} - \Pi T \hat{v} ||_{\Xi}^{2}}.
\end{align}

\subsection{On-policy Comparison}

\begin{figure}[tbh]
\centering
\includegraphics[width=.47\textwidth,height=1.915in]{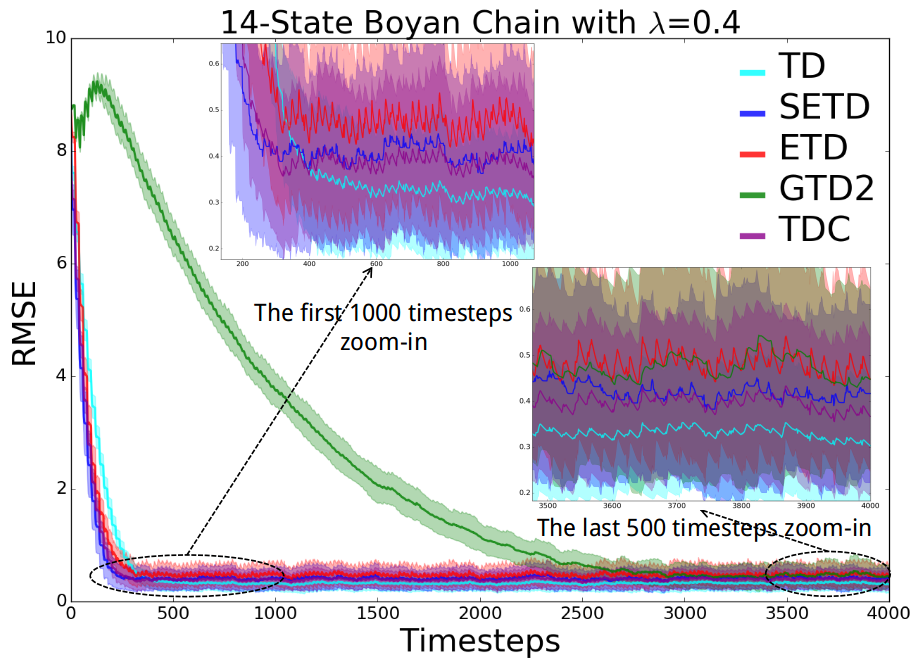}
\includegraphics[width=.47\textwidth,height=1.915in]{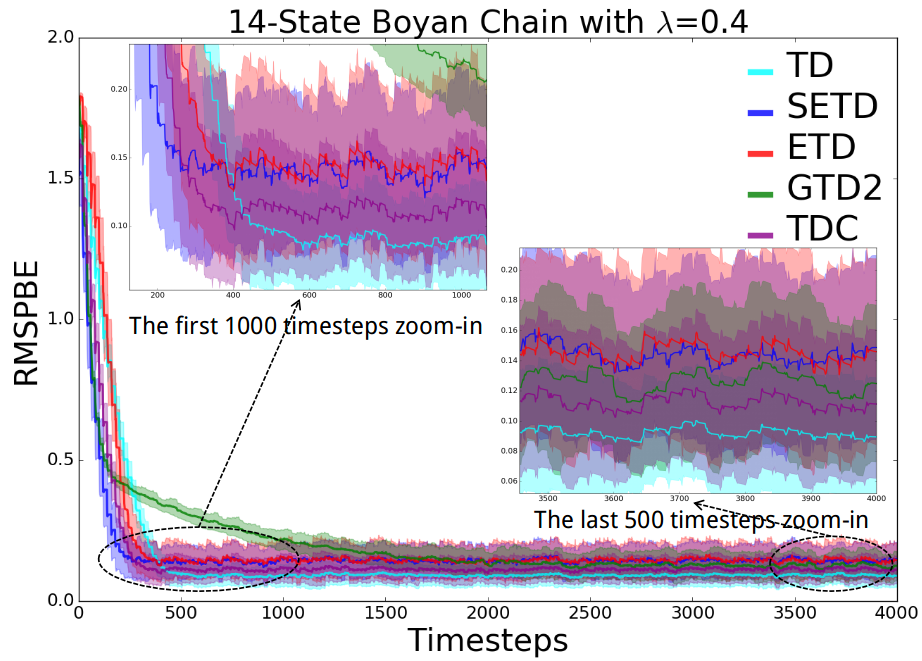}
\caption{14-state Boyan Chain MDP, $\lambda=0.4$.}
\label{fig:boyanchain-0.4}
\end{figure}

\begin{figure}[tbh]
\centering 
\includegraphics[width=.47\textwidth,height=1.915in]{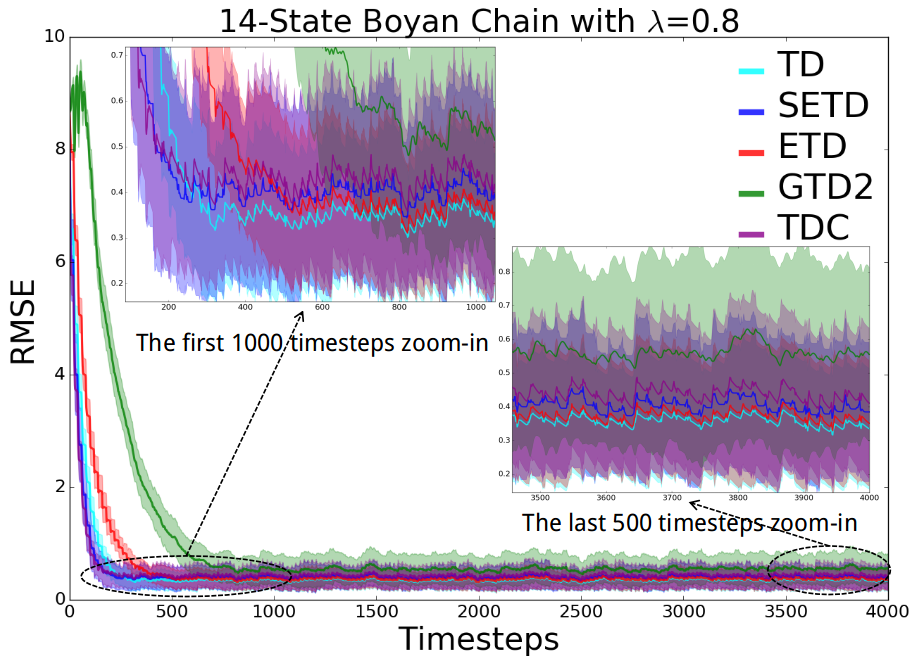}
\includegraphics[width=.47\textwidth,height=1.915in]{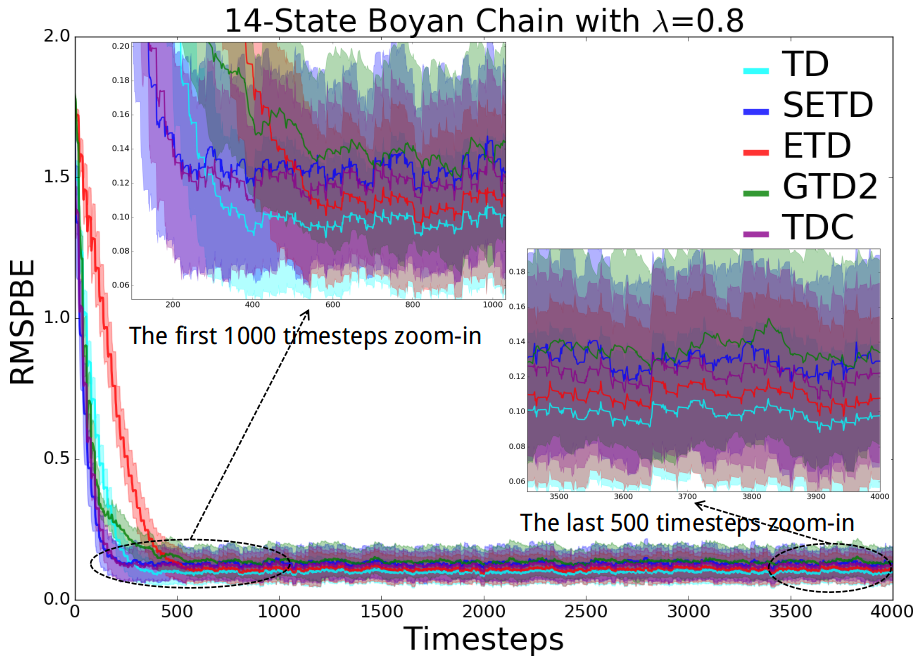}
\caption{14-state Boyan Chain MDP, $\lambda=0.8$.}
\label{fig:boyanchain-0.8}
\end{figure}

\subsubsection{Boyan Chain}
Comparison studies are conducted on the Boyan Chain MDP, which has 14 states and one action with a 4-dimensional state representation~\citep{boyan2002technical}. 
Algorithm 2 is compared with $\lambda = 0.4, 0.8$, as shown in Figure~\ref{fig:boyanchain-0.4} and \ref{fig:boyanchain-0.8} respectively.
To enable the visibility of details, we zoom in the first 1000 timesteps and the last 500 timesteps in the figures. For $\lambda=0.4$, constant stepsizes are $\alpha_{TD}=0.2$, $\alpha_{SETD}=0.4$, $\alpha_{ETD}=0.04$, $\alpha_{GTD2}=0.5$, $\mu_{GTD2}=1$,$\alpha_{TDC}=0.3$, $\mu_{TDC}=0.001$. For $\lambda=0.8$, constant stepsizes are $\alpha_{TD}=0.2$, $\alpha_{SETD}=0.3$, $\alpha_{ETD}=0.04$, $\alpha_{GTD2}=0.3$, $\mu_{GTD2}=1$,$\alpha_{TDC}=0.3$, $\mu_{TDC}=0.001$. 
 The learning curves are averaged over the results of $20$ runs.
Compared with all of the other approaches, SETD tends to have the fastest convergence speed and reaches similar steady-state performance on both RMSE and RMSPBE.

\begin{figure}
\centering
\includegraphics[width=.47\textwidth,height=1.92in]{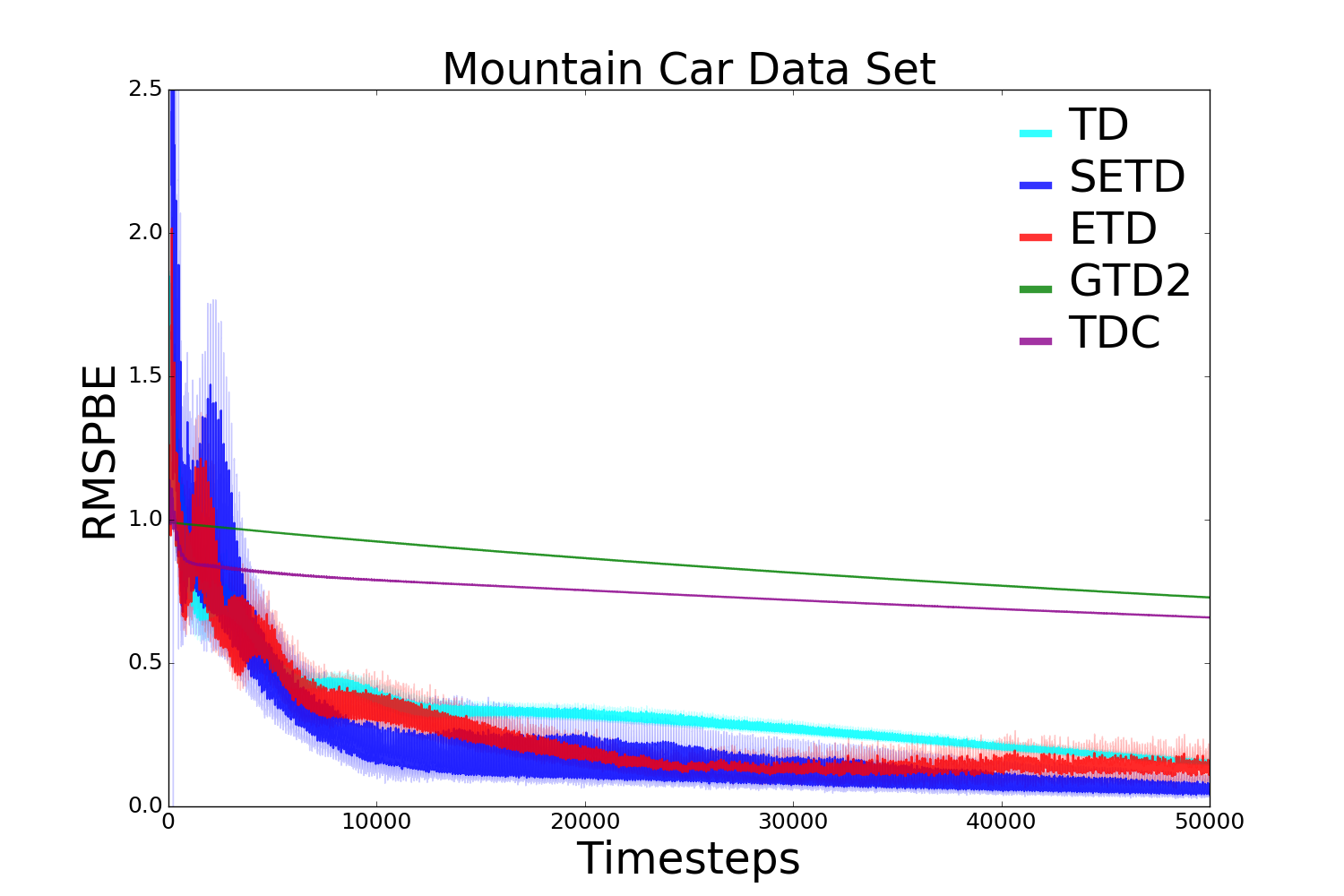}
\includegraphics[width=.47\textwidth,height=1.92in]{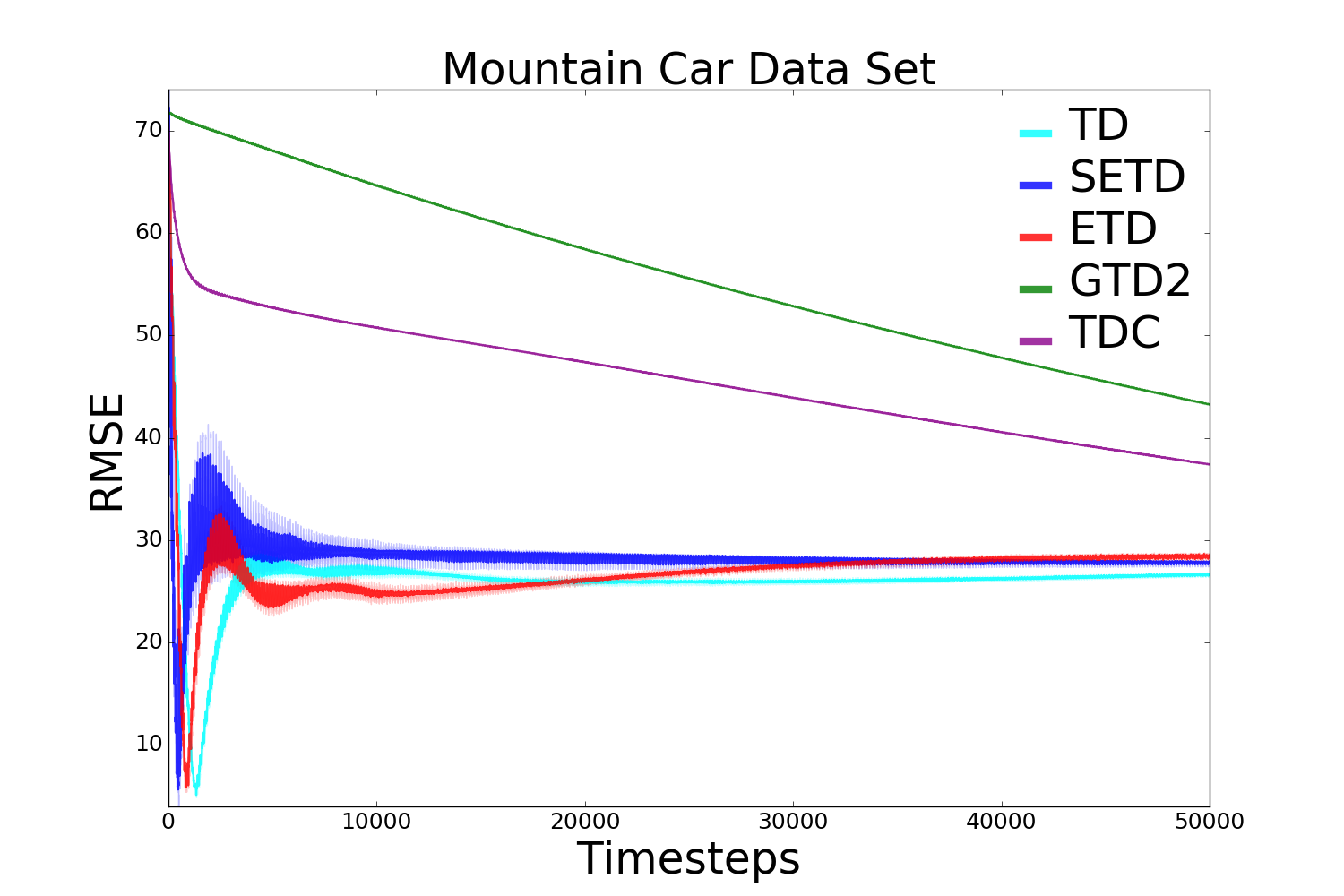}
\caption{Mountain Car with On-policy Task}
\label{fig:mcar-seq}
\end{figure}

\subsubsection{Mountain Car}
The mountain car problem is also used to evaluate the validity of SETD.
The mountain car MDP is an optimal control problem with a continuous two-dimensional state space. The steep discontinuity in the value function makes learning difficult. The Fourier basis \citep{konidaris:fourier} is used, which is a kind of fixed basis set.
In this experiment, an empirically good policy $\pi$ was first obtained, then we ran this policy $\pi$ to collect trajectories that comprise the dataset. On-policy policy evaluation of $\pi$ is then conducted using the collected samples.
The constant stepsizes are chosen as  $\alpha_{TD}=0.1$, $\alpha_{{ETD}} = 0.002$, $\alpha_{{SETD}} = 0.4$, $\alpha_{{GTD2}} = 0.05$, $\mu_{{GTD2}} = 1$, $\alpha_{{TDC}} = 0.04$, $\mu_{{TDC}} = 0.05$.
 The learning curves are averaged over the results of $20$ runs.
The Monte-Carlo estimation of $V$ is estimated via $100$ runs and each run has a maximum of  
$200$ time steps.

As Figure~\ref{fig:mcar-seq} shows, TD performs slightly better than SETD and ETD. TDC and GTD2 converge slowly in this domain.

\subsection{Off-policy Comparison}

\begin{figure}[tbh]
\centering
\includegraphics[width=.241\textwidth,height=1.2in]{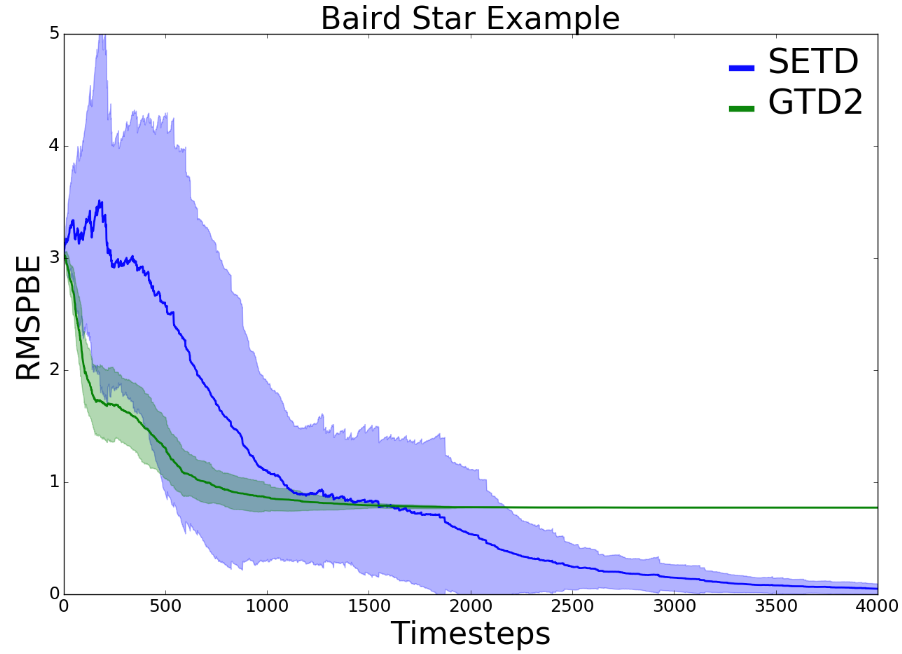}
\includegraphics[width=.241\textwidth,height=1.2in]{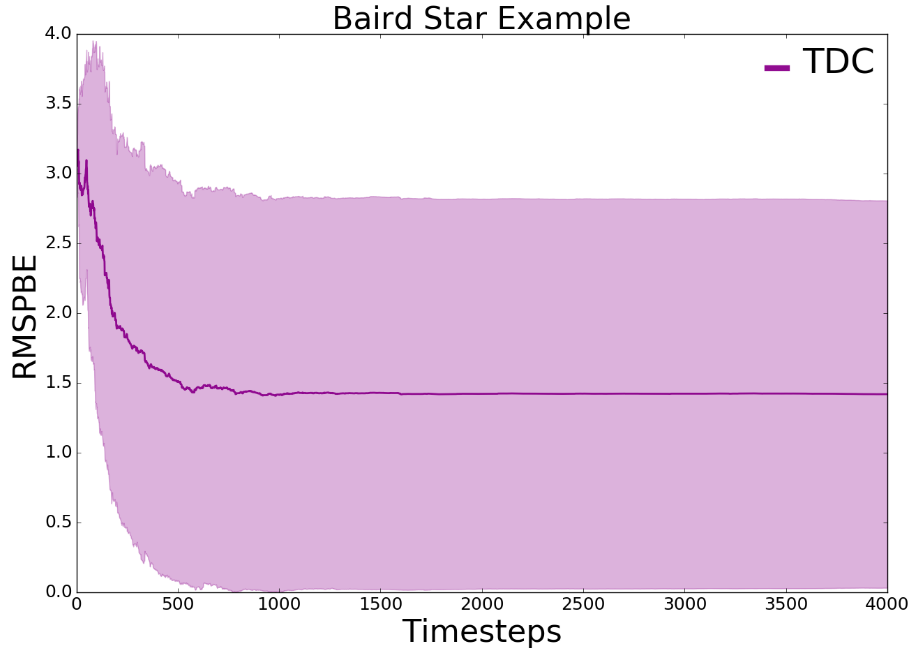}
\caption{Off-policy Baird Domain (RMSPBE)}
\label{fig:baird-mspbe}
\end{figure}

\begin{figure}[tbh]
\centering
\includegraphics[width=.241\textwidth,height=1.2in]{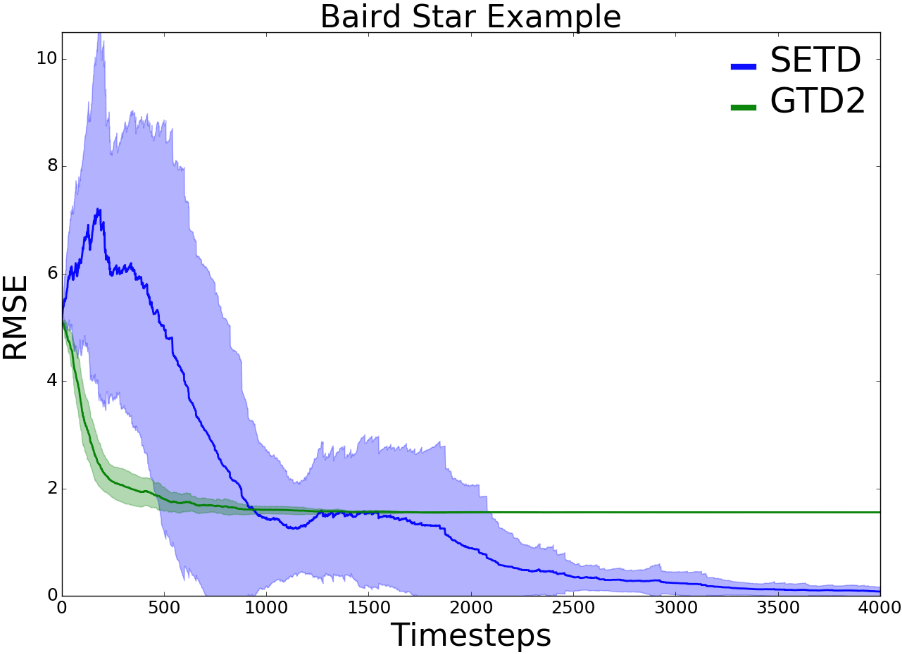}
\includegraphics[width=.241\textwidth,height=1.2in]{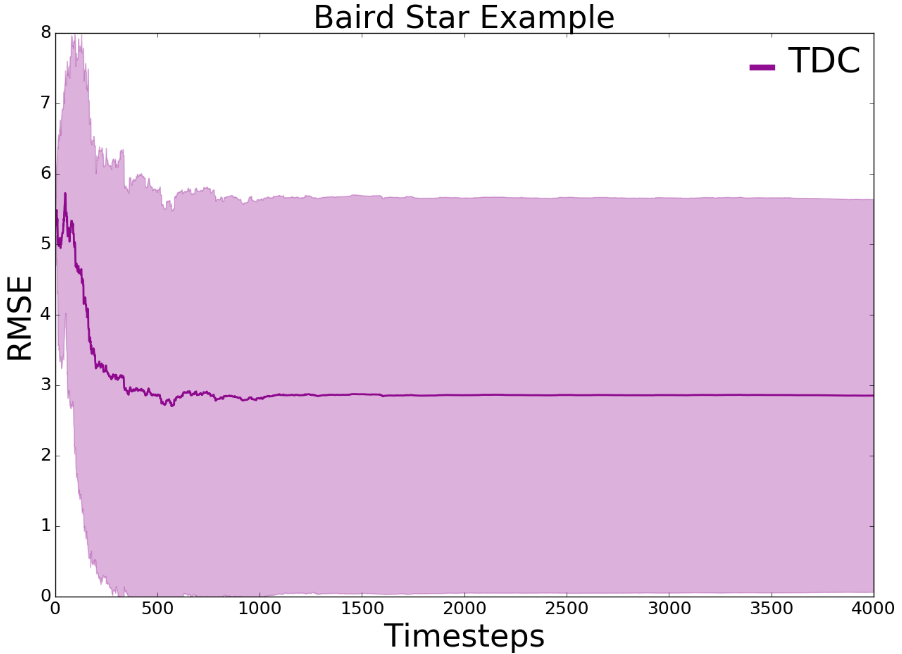}
\caption{Off-policy Baird Domain (RMSE)}
\label{fig:baird-mse}
\end{figure}

\subsubsection{Baird Domain}
The Baird example \citep{Baird:ResidualAlgorithms1995} is a well-known example to test the performance of off-policy convergent algorithms. 
Constant stepsizes $\alpha_{{SETD}} = 0.006$, $\alpha_{{GTD2}} = 0.005$, $\mu_{{GTD2}} = 1$, $\alpha_{{TDC}} = 0.006$, $\mu_{{TDC}} = 16$, which are chosen via comparison studies as in \citep{dann2014tdsurvey}.


Figure~\ref{fig:baird-mspbe} and Figure~\ref{fig:baird-mse} show the RMSPBE curve and RMSE curve of GTD2, SETD, TDC of $4000$ steps averaged over $20$ runs. 
TDC has the largest variance of all; although the variance of SETD is larger than GTD2's, SETD has a significant improvement over the GTD2 algorithm wherein the RMSPBE, the RMSE, and the variance are all substantially reduced.
The low variance of the GTD2 learning curve can be explained by the advantage of the stochastic gradient method~\cite{liu2015uai}.

\begin{figure}[tbh]
\centering
\includegraphics[width=.47\textwidth,height=1.91in]{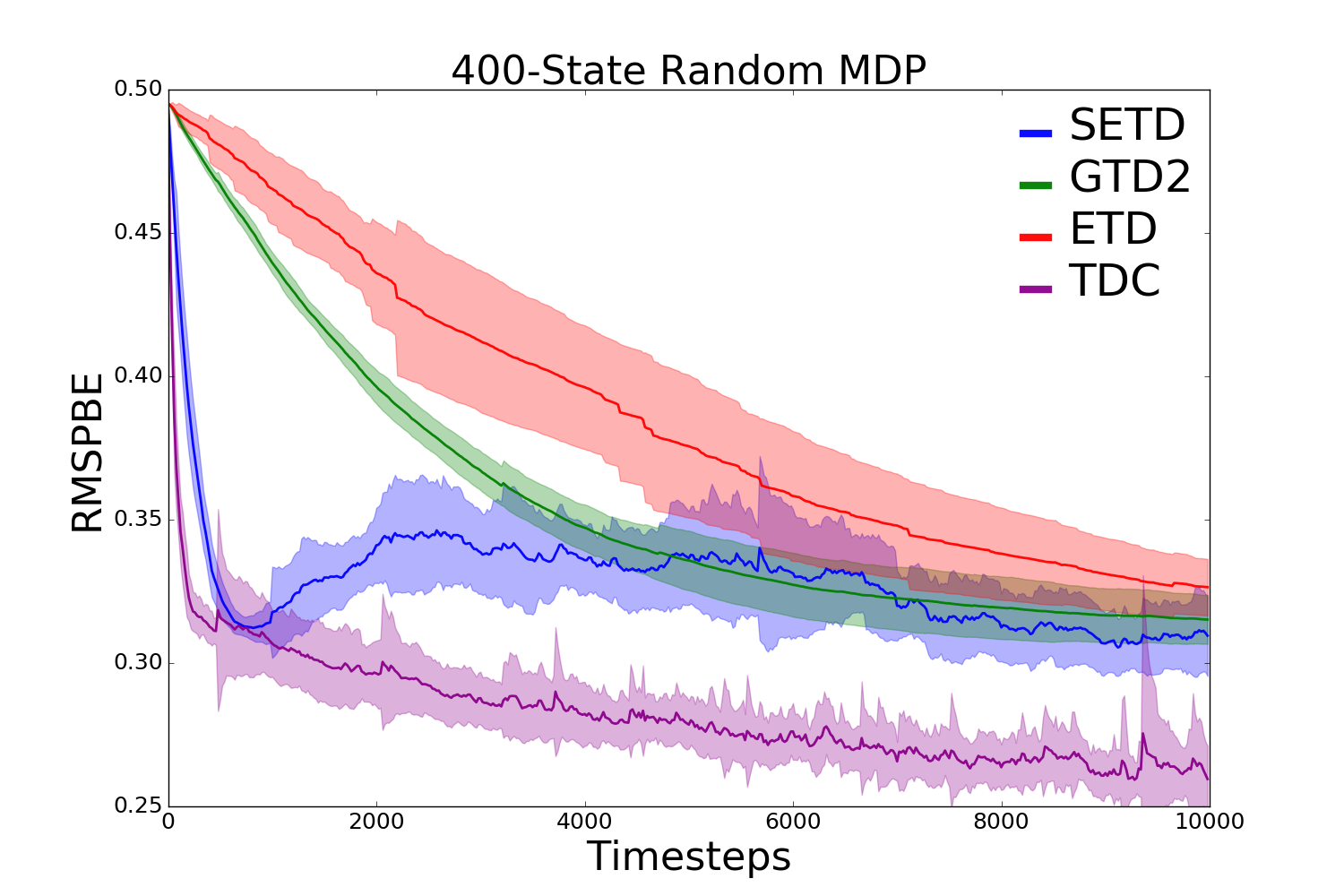}
\includegraphics[width=.47\textwidth,height=1.91in]{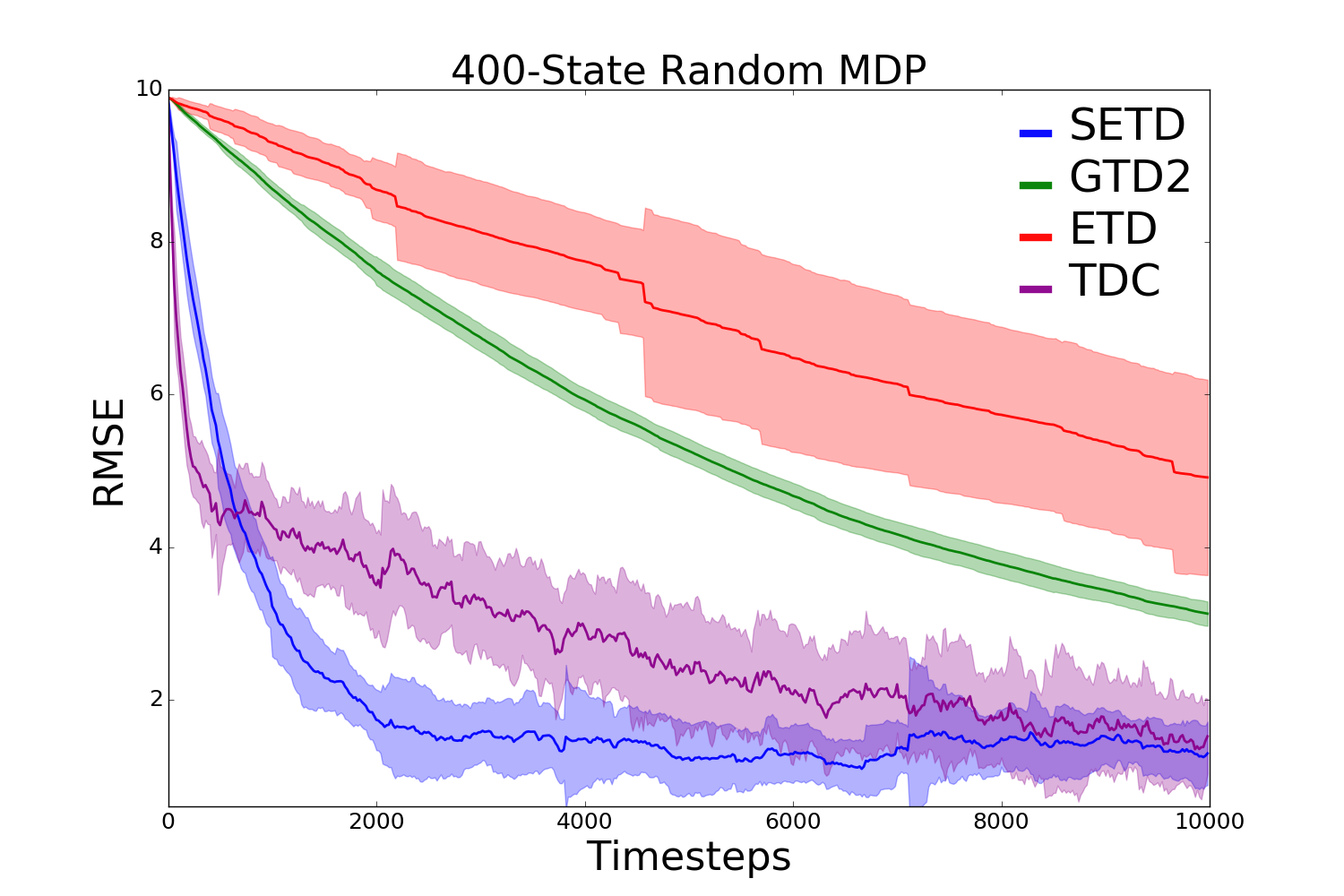}
\caption{Random MDP with Off-policy Task}
\label{fig:ran-seq}
\end{figure}

\subsubsection{$400$-State Random MDP}
%
This domain is a randomly generated MDP with $400$ states and $10$ actions \citep{dann2014tdsurvey}. The transition probabilities are defined as $P(s'|s,a) \propto p_{ss'}^a + {10^{ - 5}}$, where 
$p_{ss'}^a \sim U[0,1]$.
The behavior policy $\pi_b$, the target policy $\pi$ as well as the starting distribution, are sampled in a similar manner. 
Each state is represented by a $201$-dimensional feature vector, where the first $200$ features were sampled from a uniform distribution, and the last feature was a constant one, the discount factor is set to $\gamma = 0.95$.
The constant stepsizes are chosen as $\alpha_{{ETD}} = 2.5*10^{-6}$, $\alpha_{{SETD}} = 0.0008$, $\alpha_{{GTD2}} = 0.002$, $\mu_{{GTD2}} = 1$, $\alpha_{{TDC}} = 0.002$, $\mu_{{TDC}} = 0.05$.
Both the RMSE curve and RMSPBE curve are averaged over 20 runs, and each run has $10,000$ time steps.
%
ETD is very sensitive to stepsizes on this domain and tends to diverge with a large stepsize, thus makes the convergence very slow. 
As Figure~\ref{fig:ran-seq} shows, SETD performs better than GTD2 and ETD, although the variance is relatively larger than GTD2's. 

Overall, although SETD tends to have a relatively large variance in the initial phase, it is (1) off-policy convergent, (2) converging much faster than GTD2 and TDC in both on-policy and off-policy settings, and (3) less sensitive to stepsizes than ETD in the off-policy setting.


\section{Conclusion}

This paper addressed the question: \textit{How to design a policy evaluation algorithm that is both off-policy stable and on-policy efficient?}
Novel algorithms have been proposed, based on oblique projection. 
Empirical experimental studies showed the effectiveness of the proposed algorithms in different learning settings.

There are numerous promising future work potentials along this direction of research. One is to conduct a sample complexity analysis such as a high probability error bound. This is important because the performance of machine learning algorithms is always evaluated with a finite number of samples in real applications. Another interesting direction is to explore new approximation criteria. Current computationally tractable criteria of computing $X^*$ are based on Proposition~\ref{pro:fundamental} (as used in SETD) or on the power series expansion of $(L_\pi^\top)^{-1}\Xi\Phi$ (as used in ETD). It is intriguing to explore if there exists other computationally tractable criteria.
Finally, how to design a true-online SETD($\lambda$) algorithm and compare its performance with other true-online algorithms~\citep{pgtd:adam2016} such as true-online GTD($\lambda$) and true-online ETD($\lambda$) is also worth exploring.


\ifCLASSOPTIONcaptionsoff
  \newpage
\fi

\bibliographystyle{IEEEtran}
\bibliography{thesisbib}

\begin{thebibliography}{10}
\providecommand{\url}[1]{#1}
\csname url@samestyle\endcsname
\providecommand{\newblock}{\relax}
\providecommand{\bibinfo}[2]{#2}
\providecommand{\BIBentrySTDinterwordspacing}{\spaceskip=0pt\relax}
\providecommand{\BIBentryALTinterwordstretchfactor}{4}
\providecommand{\BIBentryALTinterwordspacing}{\spaceskip=\fontdimen2\font plus
\BIBentryALTinterwordstretchfactor\fontdimen3\font minus
  \fontdimen4\font\relax}
\providecommand{\BIBforeignlanguage}[2]{{%
\expandafter\ifx\csname l@#1\endcsname\relax
\typeout{** WARNING: IEEEtran.bst: No hyphenation pattern has been}%
\typeout{** loaded for the language `#1'. Using the pattern for}%
\typeout{** the default language instead.}%
\else
\language=\csname l@#1\endcsname
\fi
#2}}
\providecommand{\BIBdecl}{\relax}
\BIBdecl

\bibitem{gordon1996stable}
G.~J. Gordon, ``Stable fitted reinforcement learning,'' \emph{Advances in
  neural information processing systems}, pp. 1052--1058, 1996.

\bibitem{yu2010:icml}
H.~Yu, ``Convergence of least squares temporal difference methods under general
  conditions,'' in \emph{Proceedings of the 27th International Conference on
  Machine Learning}, 2010, pp. 1207--1214.

\bibitem{tdc:2009}
R.~Sutton, H.~Maei, D.~Precup, S.~Bhatnagar, D.~Silver, C.~Szepesv{\'a}ri, and
  E.~Wiewiora, ``Fast gradient-descent methods for temporal-difference learning
  with linear function approximation,'' in \emph{International Conference on
  Machine Learning}, 2009, pp. 993--1000.

\bibitem{liu2015uai}
B.~Liu, J.~Liu, M.~Ghavamzadeh, S.~Mahadevan, and M.~Petrik, ``Finite-sample
  analysis of proximal gradient td algorithms,'' in \emph{Conference on
  Uncertainty in Artificial Intelligence}, 2015.

\bibitem{pgtd:adam2016}
A.~White and M.~White, ``Investigating practical linear temporal difference
  learning,'' in \emph{Proceedings of the 2016 International Conference on
  Autonomous Agents \& Multiagent Systems}, 2016, pp. 494--502.

\bibitem{Scherrer:ObliqueProjection}
B.~Scherrer, ``Should one compute the temporal difference fix point or minimize
  the bellman residual? the unified oblique projection view,'' in
  \emph{Proceedings of 27 th International Conference on Machine Learning},
  2010, pp. 52--68.

\bibitem{sutton-barto:book}
R.~Sutton and A.~G. Barto, \emph{{Reinforcement Learning: An
  Introduction}}.\hskip 1em plus 0.5em minus 0.4em\relax MIT Press, 1998.

\bibitem{tnnls:kiumarsi:2018}
B.~Kiumarsi, K.~G. Vamvoudakis, H.~Modares, and F.~L. Lewis, ``Optimal and
  autonomous control using reinforcement learning: A survey,'' \emph{IEEE
  transactions on neural networks and learning systems}, vol.~29, no.~6, pp.
  2042--2062, 2018.

\bibitem{tnnls:liu:2014policy}
D.~Liu and Q.~Wei, ``Policy iteration adaptive dynamic programming algorithm
  for discrete-time nonlinear systems,'' \emph{IEEE transactions on neural
  networks and learning systems}, vol.~25, no.~3, pp. 621--634, 2014.

\bibitem{tnnls:he:2017:survey}
D.~Wang, H.~He, and D.~Liu, ``Adaptive critic nonlinear robust control: A
  survey,'' \emph{IEEE transactions on cybernetics}, vol.~47, no.~10, pp.
  3429--3451, 2017.

\bibitem{saad:book}
Y.~Saad, \emph{Iterative Methods for Sparse Linear Systems}.\hskip 1em plus
  0.5em minus 0.4em\relax SIAM Press, 2003.

\bibitem{Baird:ResidualAlgorithms1995}
L.~C. Baird, ``Residual algorithms: Reinforcement learning with function
  approximation,'' in \emph{International Conference on Machine Learning},
  1995, pp. 30--37.

\bibitem{tsitsiklis-roy:tdfun}
J.~Tsitsiklis and B.~Van~Roy, ``An analysis of temporal-difference learning
  with function approximation,'' \emph{IEEE Transactions on Automatic Control},
  vol.~42, pp. 674--690, 1997.

\bibitem{dann2014tdsurvey}
C.~Dann, G.~Neumann, and J.~Peters, ``Policy evaluation with temporal
  differences: A survey and comparison,'' \emph{Journal of Machine Learning
  Research}, vol.~15, pp. 809--883, 2014.

\bibitem{geist2014off:trace:jmlr}
M.~Geist and B.~Scherrer, ``Off-policy learning with eligibility traces: a
  survey,'' \emph{The Journal of Machine Learning Research}, vol.~15, no.~1,
  pp. 289--333, 2014.

\bibitem{etd:sutton2015}
R.~S. Sutton, A.~R. Mahmood, and M.~White, ``An emphatic approach to the
  problem of off-policy temporal-difference learning,'' \emph{The Journal of
  Machine Learning Research}, vol.~17, pp. 1--29, 2015.

\bibitem{saferl:munos2016safe}
R.~Munos, T.~Stepleton, A.~Harutyunyan, and M.~G. Bellemare, ``Safe and
  efficient off-policy reinforcement learning,'' in \emph{Advances in Neural
  Information Processing Systems}, 2016.

\bibitem{precup:tb:2000}
D.~Precup, R.~S. Sutton, and S.~P. Singh, ``{Eligibility Traces for Off-Policy
  Policy Evaluation},'' in \emph{International Conference on Machine Learning},
  2000, pp. 759--766.

\bibitem{yu2015convergence}
H.~Yu, ``On convergence of emphatic temporal-difference learning,'' in
  \emph{Conference on Learning Theory}, 2015, pp. 1724--1751.

\bibitem{mahmood2017incremental}
A.~Mahmood, ``Incremental off-policy reinforcement learning algorithms,'' Ph.D.
  dissertation, University of Alberta, 2017.

\bibitem{boyan2002technical}
J.~A. Boyan, ``Technical update: Least-squares temporal difference learning,''
  \emph{Machine Learning}, vol.~49, no.~2, pp. 233--246, 2002.

\bibitem{konidaris:fourier}
G.~Konidaris, S.~Osentoski, and P.~S. Thomas, ``{Value function approximation
  in reinforcement learning using the Fourier basis},'' in \emph{Proceedings of
  the Twenty-Fifth Conference on Artificial Intelligence}, 2011.

\end{thebibliography}

\begin{IEEEbiography}[{\includegraphics[width=25mm,height=32mm,clip,keepaspectratio]{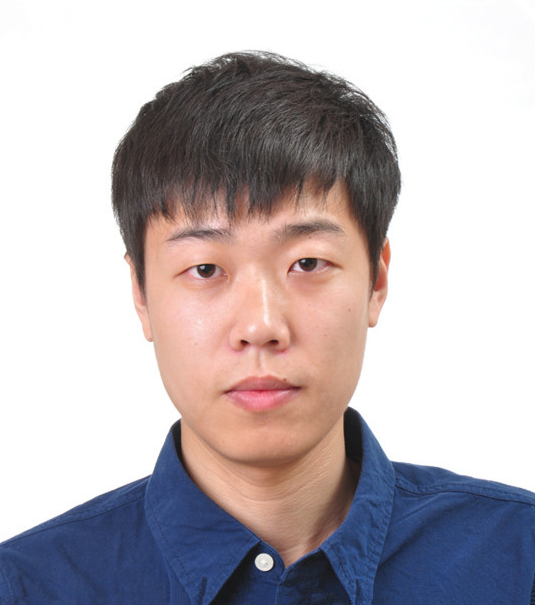}}]{Daoming Lyu}
received his B. S. degree in Electrical Engineering from Southwest University, and M. E. degree in Biomedical Engineering from Zhejiang University, China in 2011 and 2015, respectively. He is currently a Ph.D.
candidate in the Department of Computer Science and Software Engineering, Auburn University, USA. 
His research interest covers reinforcement learning, symbolic planning, healthcare informatics and artificial intelligence. His website is \url{http://www.auburn.edu/~dzl0053/}.
\end{IEEEbiography}

\begin{IEEEbiography}[{\includegraphics[width=25mm,height=32mm,clip,keepaspectratio]{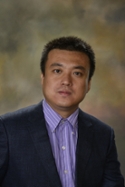}}]{Bo Liu} (M'18)
is a tenure-track assistant professor in Dept. of Computer Science and Software Engineering at Auburn University. He obtained his Ph.D. in University of Massachusetts Amherst, 2015. His primary research area covers machine learning, deep learning, healthcare informatics, stochastic optimization and their numerous applications to BIGDATA. In his current research, he has more than 30 publications on several notable venues, such as NIPS, UAI, AAAI, IJCAI, JAIR, IEEE TNNLS, ACM TECS, etc. He is the recipient of the UAI-2015 Facebook best student paper award. His website is \url{http://www.eng.auburn.edu/~bzl0056/}.
\end{IEEEbiography}

\begin{IEEEbiography}[{\includegraphics[width=25mm,height=32mm,clip,keepaspectratio]{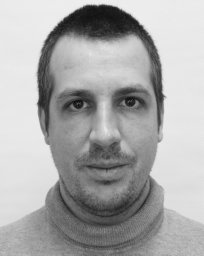}}]{Matthieu Geist}
obtained an Electrical Engineering degree and an MSc degree in Applied Mathematics in Sept. 2006 (Sup\'elec, France), a PhD degree in Applied Mathematics in Nov. 2009 (University Paul Verlaine of Metz, France) and a Habilitation degree in Feb. 2016 (University Lille 1, France). Between Feb. 2010 and Sept. 2017, he was an assistant professor at CentraleSup\'elec, France. In Sept. 2017, he joined University of Lorraine, France, as a full professor in Applied Mathematics (Interdisciplinary Laboratory for Continental Environments, CNRSL-UL). Since Sept. 2018, he is at Google Brain (Paris, France). His research interests include machine learning, especially reinforcement learning and imitation learning, as well as various applications.
\end{IEEEbiography}

\begin{IEEEbiography}[{\includegraphics[width=25mm,height=32mm,clip,keepaspectratio]{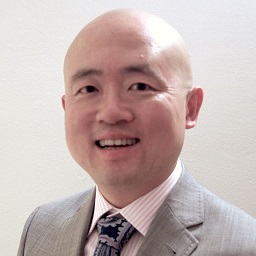}}]{Wen Dong}
is an Assistant Professor of Computer Science and Engineering at the State University of New York at Buffalo with a joint appointment in the Institute of Sustainable Transportation and Logistics. He focuses on modeling human interaction dynamics with stochastic process theory through combining the power of "big data" and the logic/reasoning power of agent-based models, to solve our societies' most challenging problems such as transportation sustainability and efficiency. Wen Dong holds a Ph.D. in Media Arts and Sciences from Massachusetts Institute of Technology.
\end{IEEEbiography}

\begin{IEEEbiography}[{\includegraphics[width=25mm,height=32mm,clip,keepaspectratio]{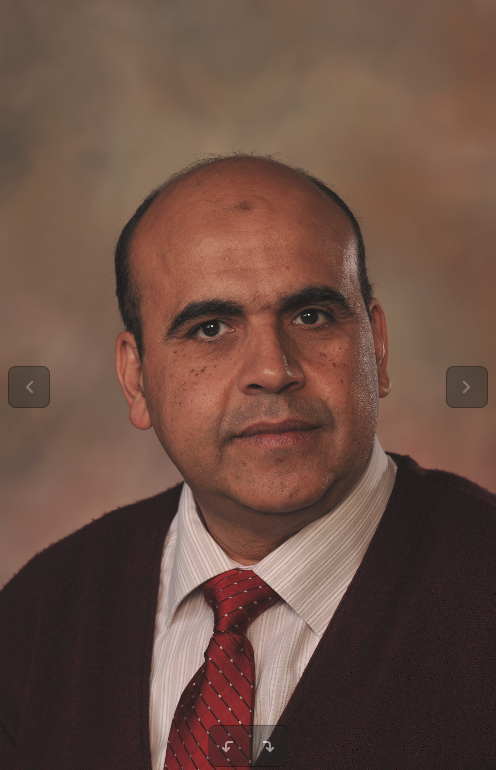}}]{Saad Biaz}
(M'98) received the Ph.D. degree in Electrical Engineering from the University Henri Poincar\'e, Nancy, France, in 1989, and the Ph.D. degree in Computer Science from Texas A\&M University, College Station, in 1999. He is presently a Professor of Computer Science and Software Engineering at Auburn University, Auburn, AL. He has held faculty positions at the Ecole Sup\'erieure de Technologie de F\`es and Al Akhawayn University, Ifrane, Morocco. His current research is in the areas of distributed systems, wireless networking, mobile computing, and unmanned flight. Dr. Biaz was a recipient of the Excellence Fulbright Scholarship in 1995. His research is funded by the National Science Foundation and the Department of Defense. He has served on the committees of several conferences and as reviewer for several journals. His website is \url{http://www.eng.auburn.edu/users/sbiaz}.
\end{IEEEbiography}
\begin{IEEEbiography}[{\includegraphics[width=25mm,height=32mm,clip,keepaspectratio]{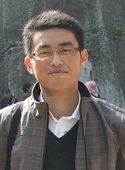}}]{Qi Wang}(M'15-SM'15) received the B.E. degree in automation and the Ph.D. degree in pattern recognition and intelligent systems from the University of Science and Technology of China, Hefei, China, in 2005  and 2010, respectively.  He is currently a Professor with the School of Computer Science, with the Unmanned System Research Institute, and with the Center for OPTical IMagery Analysis and Learning (OPTIMAL), Northwestern Polytechnical University, Xi'an, China. His research interests include computer vision and pattern recognition.
 
\end{IEEEbiography}

\end{document}